%% file: main.tex
\documentclass[sigconf]{acmart}
\usepackage{subfigure}
\usepackage{multirow}
\usepackage{pifont}
\usepackage{bbding}
\usepackage{adjustbox}
\usepackage{threeparttable}
\usepackage{colortbl}

\usepackage{bbm}

\usepackage{amsthm}
\newtheorem{definition}{Definition}

\newtheorem{lemma}{Lemma}
\newtheorem{theorem}{Theorem}

\newcommand{\vpara}[1]{\vspace{0.05in}\noindent\textbf{#1 }}

\AtBeginDocument{%
  }

\copyrightyear{2026}
\acmYear{2026}
\setcopyright{cc}
\setcctype{by-nc-nd}
\acmConference[WWW '26]{Proceedings of the ACM Web Conference 2026}{April 13--17, 2026}{Dubai, United Arab Emirates}
\acmBooktitle{Proceedings of the ACM Web Conference 2026 (WWW '26), April 13--17, 2026, Dubai, United Arab Emirates}
\acmPrice{}
\acmDOI{10.1145/3774904.3792454}
\acmISBN{979-8-4007-2307-0/2026/04}

\begin{document}

\title{Mitigating Homophily Disparity in Graph Anomaly Detection: \\A Scalable and Adaptive Approach}

\author{Yunhui Liu}
\orcid{0009-0006-3337-0886}
\affiliation{
\institution{State Key Laboratory for Novel Software Technology\\ Nanjing University}
\city{Nanjing}
\country{China}}
\email{lyhcloudy1225@gmail.com}

\author{Qizhuo Xie}
\orcid{0009-0002-5553-1768}
\affiliation{
\institution{State Key Laboratory for Novel Software Technology\\ Nanjing University}
\city{Nanjing}
\country{China}}
\email{mikumifa@gmail.com}

\author{Yinfeng Chen}
\orcid{0009-0005-6873-8786}
\author{Xudong Jin}
\orcid{0009-0000-0306-4382}
\affiliation{
\institution{State Key Laboratory for Novel Software Technology\\ Nanjing University}
\city{Nanjing}
\country{China}}

\author{Tao Zheng}
\orcid{0009-0001-3736-4604}
\affiliation{
\institution{State Key Laboratory for Novel Software Technology\\ Nanjing University}
\city{Nanjing}
\country{China}}
\email{zt@nju.edu.cn}

\author{Bin Chong}
\authornote{Corresponding authors.}
\orcid{0000-0002-8741-178X}
\affiliation{
\institution{National Engineering Laboratory for Big Data Analysis and Applications\\ Peking University}
\city{Beijing}
\country{China}}
\email{chongbin@pku.edu.cn}

\author{Tieke He}
\authornotemark[1]
\orcid{0000-0001-9649-1796}
\affiliation{
\institution{State Key Laboratory for Novel Software Technology\\ Nanjing University}
\city{Nanjing}
\country{China}}
\email{hetieke@gmail.com}

\renewcommand{\shortauthors}{Yunhui Liu et al.}

\begin{abstract}
Graph anomaly detection (GAD) aims to identify nodes that deviate from normal patterns in structure or features.
While recent GNN-based approaches have advanced this task, they struggle with two major challenges: 1) homophily disparity, where nodes exhibit varying homophily at both class and node levels; and 2) limited scalability, as many methods rely on costly whole-graph operations. 
To address them, we propose SAGAD, a Scalable and Adaptive framework for GAD.
SAGAD precomputes multi-hop embeddings and applies reparameterized Chebyshev filters to extract low- and high-frequency information, enabling efficient training and capturing both homophilic and heterophilic patterns.
To mitigate node-level homophily disparity, we introduce an Anomaly Context-Aware Adaptive Fusion, which adaptively fuses low- and high-pass embeddings using fusion coefficients conditioned on Rayleigh Quotient-guided anomalous subgraph structures for each node.
To alleviate class-level disparity, we design a Frequency Preference Guidance Loss, which encourages anomalies to preserve more high-frequency information than normal nodes.
SAGAD supports mini-batch training, achieves linear time and space complexity, and drastically reduces memory usage on large-scale graphs.
Theoretically, SAGAD ensures asymptotic linear separability between normal and abnormal nodes under mild conditions. Extensive experiments on 10 benchmarks confirm SAGAD's superior accuracy and scalability over state-of-the-art methods. 
Our code is publicly available at \url{https://github.com/Cloudy1225/SAGAD}. 
\end{abstract}




\maketitle

\section{Introduction}\label{Sec: Introduction}
With the explosive growth of Web information, graph anomaly detection (GAD), which aims to identify abnormal nodes that deviate significantly from the majority~\cite{GADSurvey}, has become a critical research direction due to its broad applications such as financial fraud detection~\cite{PMP}, malicious behavior identification in social networks~\cite{RLA}, and risk monitoring on e-commerce platforms~\cite{GAS}.

\begin{figure*}[ht]
    \centering
    \subfigure[Weibo]{\includegraphics[width=0.245\linewidth]{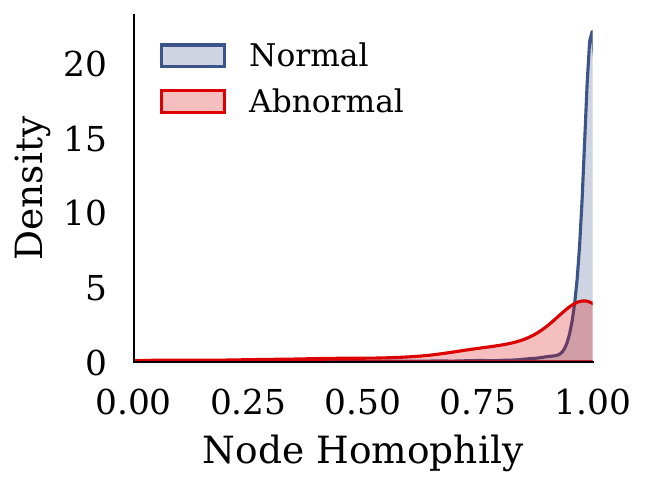}\label{Fig: Homo_Disparity_Weibo}}
    \subfigure[T-Finance]{\includegraphics[width=0.245\linewidth]{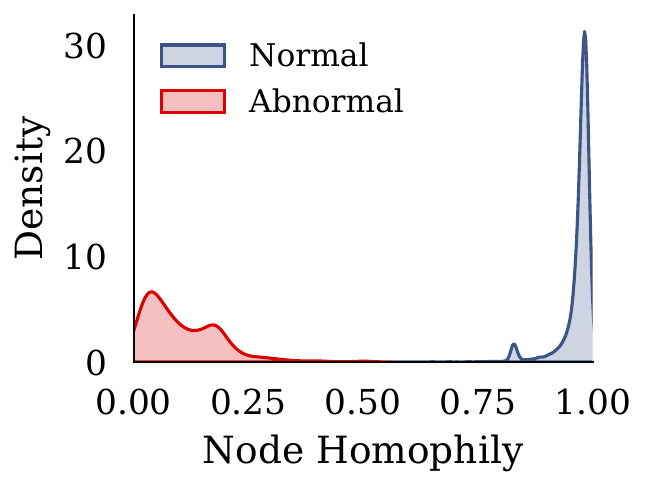}\label{Fig: Homo_Disparity_Tfinance}}
    \subfigure[Weibo]{\includegraphics[width=0.245\linewidth]{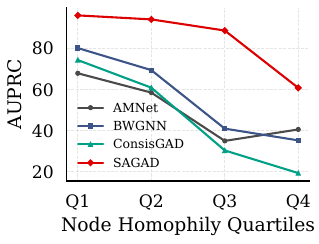}\label{Fig: Perf_Disparity_Weibo}}
    \subfigure[T-Finance]{\includegraphics[width=0.245\linewidth]{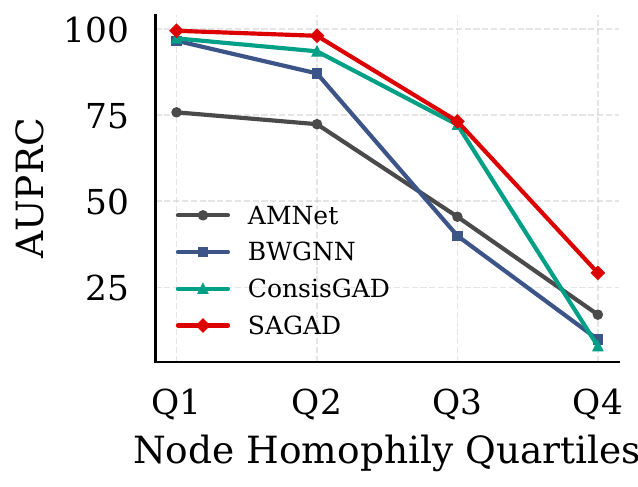}\label{Fig: Perf_Disparity_Tfinance}}
    \caption{(a), (b): Distribution of node homophily on Weibo and T-Finance. (c), (d): Performance disparity across node homophily quartiles (Q1 = top 25\% homophily, Q4 = bottom 25\%) on Weibo and T-Finance.}
    \label{Fig: Homo_Perf_Disparity}
\end{figure*}

Due to their strong capability to capture structural information, graph neural network (GNN)-based approaches have demonstrated significant advancements in GAD. 
A key challenge extensively studied in this domain is the presence of heterophilic connections (i.e., abnormal-normal edges), where anomalies tend to obscure their identity by forming connections with many normal nodes~\cite{CARE-GNN}. 
Under such circumstances, traditional GNNs struggle to distinguish abnormal nodes from normal ones, as message passing~\cite{GCN} tends to smooth node representations, thereby diluting anomaly signals and leading to suboptimal detection performance.
To mitigate this issue, various GNN-based strategies have been proposed. These can be broadly categorized into two families: spatial-centric and spectral-centric. 
Spatial-centric approaches aim to suppress noise from heterophilic neighbors during aggregation by leveraging trainable or predefined mechanisms, such as neighbor resampling~\cite{PCGNN}, edge perturbation~\cite{GHRN}, or message partitioning~\cite{PMP}.
Spectral-centric approaches~\cite{AMNet, BWGNN, DSGAD}, on the other hand, focus on designing GNNs equipped with spectral filters that operate across different frequency bands to simultaneously capture the spectral characteristics of both normal and abnormal nodes.
While both strategies have achieved promising results, they still suffer from two critical limitations that hinder their applicability in practice.

\textbf{Homophily Disparity.} 
Existing methods primarily analyze global homophily, i.e., the overall proportion of heterophilic edges in the graph. 
Their designs are essentially one-size-fits-all, applying global operations without accounting for the diverse homophily distributions across nodes. 
As illustrated in Figures~\ref{Fig: Homo_Disparity_Weibo} and~\ref{Fig: Homo_Disparity_Tfinance}, we provide a more granular analysis of homophily patterns in GAD graphs and uncover two distinct levels of homophily disparity: 
(i) class-level disparity, where abnormal nodes typically exhibit significantly lower homophily than normal nodes; 
and (ii) node-level disparity, where local node homophily varies widely from node to node. 
These disparities pose additional challenges for anomaly detection. 
While prior methods do consider heterophily, they typically rely on a uniform design based on a global homophily assumption, lacking \emph{node-adaptive mechanisms} required to accommodate local structural diversity. 
As shown in Figures~\ref{Fig: Perf_Disparity_Weibo} and~\ref{Fig: Perf_Disparity_Tfinance}, the performance of state-of-the-art models varies substantially across node homophily quartiles, with low-homophily nodes being detected much less accurately than high-homophily ones. 
This gap underscores the urgent need for adaptive designs that explicitly account for homophily disparity.

\textbf{Limited Scalability.} 
Scalability has long been a central challenge in GNN research. 
Modern Web-scale graphs often contain millions of nodes and edges, easily surpassing the memory limits of GPUs and rendering many existing GAD models impractical. 
The challenge is particularly pronounced for GAD because heterophily-oriented designs, such as edge perturbations or learnable spectral filters, commonly depend on non-local operations over the entire graph. 
Such whole-graph dependencies cause both computational time and memory costs to grow rapidly with graph size. 
Although one might reduce hidden dimensions aggressively (e.g., setting them to 8) to enable training on a million-scale dataset such as T-Social (5,781,065 nodes, 73,105,508 edges),this significantly compromises performance, as shown in Table~\ref{Tab: AUPRC}. 
Therefore, it is crucial to develop GAD models that are scalable to large graphs while retaining the capability for anomaly detection.

To tackle these challenges, we propose SAGAD, a Scalable and Adaptive Graph Anomaly Detection framework. 
SAGAD emphasizes simplicity and scalability by decoupling graph structure from iterative computations and relying solely on a set of precomputed node embeddings. 
Specifically, we perform multi-hop graph propagation as a preprocessing step and cache the resulting embeddings. 
These embeddings are then reweighted using two sets of reparameterized Chebyshev polynomial coefficients to generate low- and high-pass embeddings, emphasizing low- and high-frequency information conducive to homophily and heterophily, respectively.
To mitigate node-level homophily disparity, we design an Anomaly Context-aware Adaptive Fusion (ACAF) that adaptively combines low- and high-pass embeddings for each node. 
ACAF incorporates the input attributes and Rayleigh Quotient-guided anomalous subgraph structures of each node to generate personalized fusion coefficients. 
To further alleviate class-level homophily disparity, we introduce a Frequency Preference Guidance Loss, which encourages abnormal nodes to retain higher-frequency information compared to normal nodes. 
On the efficiency side, SAGAD supports straightforward minibatching, enjoys lower training time, and achieves superior memory utilization compared with existing models.
Theoretically, we show that under mild conditions, SAGAD achieves linear separability between both normal and abnormal nodes. 
Extensive experiments on 10 benchmark datasets validate the effectiveness and scalability of SAGAD, establishing new state-of-the-art performance on both small and large-scale graphs. 



\section{Related Work}

\subsection{Graph Anomaly Detection}
GNN-based approaches have emerged as a promising paradigm for GAD, owing to their powerful capability to model both structural and attribute information~\cite{GADSurvey}. 
Unsupervised GAD methods~\cite{DOMINANT, CoLA, TAM, DiffGAD, SmoothGNN} typically operate under the assumption that no label information is available and detect anomalies via reconstruction, contrastive learning, etc. 
In contrast, supervised GAD methods formulate the task as a binary node classification problem and design specialized GNNs to enhance performance. 
For instance, AMNet~\cite{AMNet} introduces a constrained Bernstein polynomial parameterization to model filters across multiple frequency groups. 
BWGNN~\cite{BWGNN} employs a Beta graph wavelet to construct band-pass filters that highlight anomaly-relevant signals. 
Models such as CARE-GNN~\cite{CARE-GNN}, PCGNN~\cite{PCGNN}, and GHRN~\cite{GHRN} focus on selectively pruning inter-class edges, either via neighbor distribution analysis or graph spectral heuristics. 
PMP~\cite{PMP} proposes a partitioned message-passing scheme to separately process homophilic and heterophilic neighborhoods. 
To address label scarcity, ConsisGAD~\cite{ConsisGAD}, GGAD~\cite{GGAD}, SpaceGNN~\cite{SpaceGNN}, and APF~\cite{APF} explore strategies such as consistency training, data synthesis, multiple latent space modeling, and pre-training with fine-tuning, respectively. 
GADBench~\cite{GADBench} offers a comprehensive benchmark for supervised GAD, along with a strong tree-based model, XGBGraph. Our proposed method builds upon GADBench and further addresses the challenges of homophily disparity and scalability in GAD.

\subsection{Heterophilic Graph Neural Networks}
A growing body of work has focused on designing GNNs that perform well under heterophily, where connected nodes often have dissimilar labels or features. 
H2GCN~\cite{H2GCN} identifies key architectural choices (such as residual ego embeddings and higher-order neighbors) that help alleviate the negative effects of heterophilic edges. 
FAGCN~\cite{FAGCN} introduces a self-gating mechanism to adaptively fuse signals across different frequency components during message passing. 
ACM~\cite{ACM} approaches heterophily from a post-aggregation similarity perspective and dynamically balances aggregation, diversification, and identity mappings within each GNN layer. 
AdaGNN~\cite{AdaGNN}, BernNet~\cite{BernNet}, and ChebNetII~\cite{ChebNet2} design polynomial spectral filters to capture task-relevant frequency information more flexibly.
In the GAD setting, several works have also incorporated ideas to mitigate the impact of heterophilic connections~\cite{AMNet, BWGNN, GHRN, PMP, DSGAD}. 
While these methods do consider the presence of heterophily, they fall short in addressing the node-level and class-level homophily disparity that often arises in real-world GAD tasks. 
Our work directly targets this gap by designing a model that adaptively resolves homophily disparity while maintaining scalability.

\begin{figure*}[ht]
\centerline{\includegraphics[width=1.\linewidth]{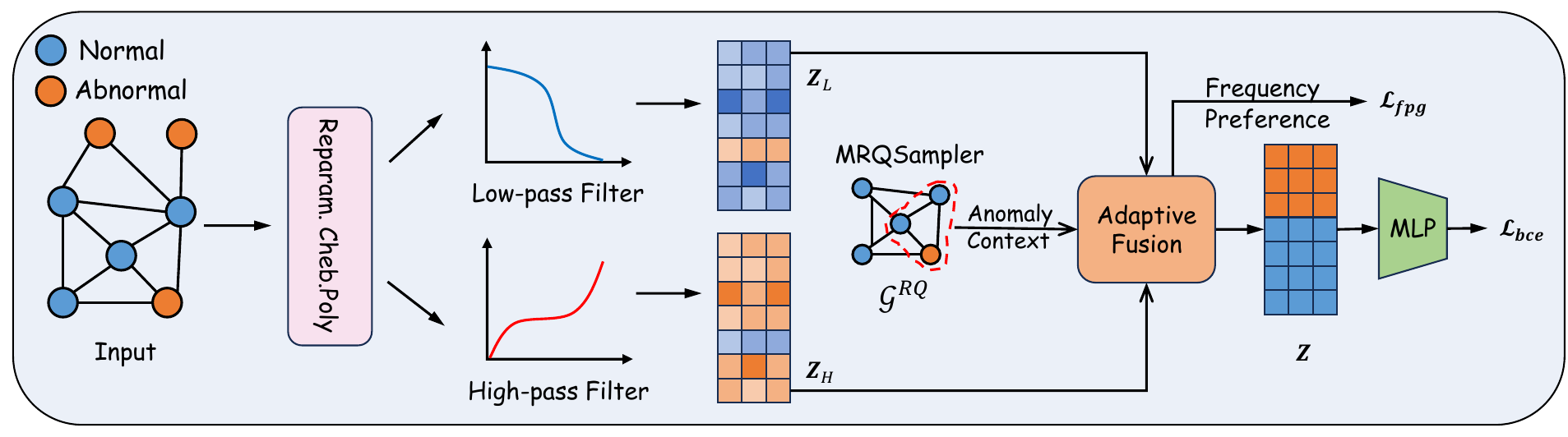}}
\caption{Overview of our proposed SAGAD. 
It consists of three main designs: 1) Dual-pass Chebyshev Polynomial Filter extracts both low- and high-frequency embeddings to capture homophilic and heterophilic patterns; 2) Anomaly Context-aware Adaptive Fusion dynamically integrates these embeddings based on node-specific structural contexts; 3) Frequency Preference Guidance Loss regularizes fusion weights to align with class-specific spectral preferences, enhancing anomaly discrimination.
} 
\label{Fig: Framework}
\end{figure*}

\section{Preliminaries}
\vpara{Graph Anomaly Detection.}
Let $\mathcal{G} = (\mathcal{V}, \mathcal{E}, \boldsymbol{X})$ denote a graph with $n$ nodes $\mathcal{V}$, $m$ edges $\mathcal{E}$, and node features $\boldsymbol{X} \in \mathbb{R}^{n \times d}$. 
Let $\boldsymbol{A} \in \mathbb{R}^{n \times n}$ denote the adjacency matrix, and $\boldsymbol{D}$ be the diagonal degree matrix. 
The set of neighbors for node $v_i$ is denoted by $\mathcal{N}_i$.
GAD is formulated as a binary classification task, where anomalies are treated as the positive class (label 1). Given a set of labeled nodes $\mathcal{V}^L = \mathcal{V}_a \cup \mathcal{V}_n$ with corresponding labels $\boldsymbol{y}^L$, the objective is to predict the labels $\boldsymbol{\hat{y}}^U$ for the unlabeled nodes.

\vpara{Graph Spectral Filtering.}
Define graph Laplacian as $\boldsymbol{L}=\boldsymbol{D}-\boldsymbol{A}$ and the symmetric normalized version as $\tilde{\boldsymbol{L}} = \boldsymbol{I} - \boldsymbol{D}^{-1/2}\boldsymbol{A}\boldsymbol{D}^{-1/2}$. 
The normalized Laplacian can be decomposed as $\tilde{\boldsymbol{L}} = \boldsymbol{U} \boldsymbol{\Lambda} \boldsymbol{U}^\top$, where $\boldsymbol{U} \in \mathbb{R}^{n \times n}$ contains the orthonormal eigenvectors (graph Fourier basis), and $\boldsymbol{\Lambda} = \mathrm{diag}(\lambda_1, \lambda_2, \dots, \lambda_n)$ is a diagonal matrix of eigenvalues (frequencies), satisfying $0 = \lambda_1 \le \lambda_2 \le \dots \le \lambda_n \le 2$.
Given features $\boldsymbol{X}$ and a filter function $g(\cdot)$, the filtered features are computed as: $\boldsymbol{\hat{X}} = \boldsymbol{U} g(\boldsymbol{\Lambda}) \boldsymbol{U}^\top \boldsymbol{X}$. 
Recent studies~\cite{BernNet, ChebNet2} suggest using polynomials to approximate $g(\boldsymbol{\Lambda})$ by $K$-order truncation can fit $g(\boldsymbol{\Lambda})$ of any shape and avoid $O(n^3)$ complexity of eigendecomposition.

\vpara{Homophily and Heterophily.}\label{APP: Homophily and Heterophily}
Homophily refers to the tendency of edges to connect nodes with the same label, whereas heterophily describes the opposite phenomenon. 
There are three commonly used homophily metrics: edge, node, and class homophily.

\textbf{Edge homophily}~\cite{H2GCN} is defined as the fraction of intra-class edges in the graph: 
\begin{equation}
    h=\frac{|e_{ij} \in \mathcal{E}: y_i=y_j|}{|\mathcal{E}|}.
\end{equation}

\textbf{Node homophily}~\cite{Geom-GCN} measures the fraction of a node's neighbors that share the same label: 
\begin{equation}
    h_i = \frac{|j\in\mathcal{N}_i: y_i=y_j|}{|\mathcal{N}_i|}. 
\end{equation}

\textbf{Class homophily}~\cite{ACM} here computes the average node homophily for abnormal and normal nodes, respectively:
\begin{equation}
    h^a = \frac{\sum_{y_i=1}h_i}{\sum_{y_i=1} 1}, \quad h^n = \frac{\sum_{y_i=0}h_i}{\sum_{y_i=0} 1}.
\end{equation}
Note that all these metrics lie within the range $[0,1]$. 
Graphs or nodes with a higher proportion of intra-class edges exhibit greater homophily. 
In Table~\ref{Tab: Dataset statistics}, we provide the values of $h$, $h^a$, and $h^n$ for each dataset. 
As presented, $h^a$ is much smaller than $h^n$, indicating that abnormal nodes exhibit significantly lower homophily compared to normal ones.

\section{Methodology}
In this section, we formally introduce our SAGAD framework, which consists of three key components: 1) a Dual-pass Chebyshev Polynomial Filter that jointly captures low- and high-frequency signals, 2) an Anomaly Context-aware Adaptive Fusion mechanism that selectively combines frequency components based on each node’s local structural context, and 3) a Frequency Preference Guidance Loss that encodes class-specific spectral biases to further enhance discriminative capability. 
The overall framework is illustrated in Figure~\ref{Fig: Framework}. 
The following subsections detail each component of our framework.

\subsection{Dual-pass Chebyshev Polynomial Filter}
As discussed in the Introduction, the relation-camouflage effect of anomalies often leads to higher heterophily compared to normal nodes~\cite{CARE-GNN}. 
Recent studies~\cite{FAGCN, ACM} have demonstrated that high-frequency information plays a crucial role in capturing heterophilic structures. 
Motivated by this, we propose to complement the conventional low-pass filter with an additional high-pass filter to enhance the encoding of anomaly-related heterophilic patterns. 
To enable expressive and scalable spectral filtering, we adopt the $K$-order Chebyshev polynomial approximation with interpolation~\cite{ChebNet2, PolyGCL}. 
The general formulation is given by $\sum_{k=0}^{K} w_k T_k(\hat{\boldsymbol{L}}) X$, where $T_k(\cdot)$ denotes the $k$-th Chebyshev polynomial, $\hat{\boldsymbol{L}} = 2\tilde{L}/{\lambda_{\max}} - I$ is the scaled Laplacian, and $w_k$ is the filter coefficient. 
Each $w_k$ is reparameterized using interpolation on Chebyshev nodes as:
\begin{equation}\label{Eq: w_k}
    w_k = \frac{2}{K+1} \sum_{j=0}^{K} \gamma_j T_k(x_j),
\end{equation}
where $x_j = \cos\left(\frac{j + 1/2}{K+1} \pi\right)$ for $j = 0, \ldots, K$ are Chebyshev nodes associated with $T_{K+1}$, and $\gamma_j$ is a learnable parameter representing the filter value at node $x_j \in [-1, 1]$~\cite{ChebNet2}. 
To explicitly model high- and low-pass filters, we impose monotonicity constraints on the filter shape through reparameterization. 
Specifically, assuming non-negative filter functions as in~\cite{BernNet}, we define the high-pass filter using a prefix sum over non-negative learnable parameters and the low-pass filter using a prefix difference~\cite{PolyGCL}:
\begin{equation}\label{Eq: gamma_i}
    \gamma_{i}^{H}=\sum_{j=0}^{i}\gamma_{j},\quad\gamma_{i}^{L}=\gamma_{0}-\sum_{j=1}^{i}\gamma_{j},i=1,\ldots,K,
\end{equation}
with $\gamma_0^H = \gamma_0^L = \gamma_0$. 
This ensures that the filter responses $f(\hat{\lambda})$ are monotonically increasing for the high-pass filter ($\gamma_i^H \leq \gamma_{i+1}^H$) and monotonically decreasing for the low-pass filter ($\gamma_i^L \geq \gamma_{i+1}^L$), thereby satisfying their respective spectral properties over $\hat{\lambda} \in [-1, 1]$.
Based on this reparameterization, we define the dual-pass Chebyshev filters as:
\begin{equation}\label{Eq: Z_L_H}
    \boldsymbol{Z}_L = \sum_{k=0}^{K} w_k^L T_k(\hat{\boldsymbol{L}}) X, \quad \boldsymbol{Z}_H = \sum_{k=0}^{K} w_k^H T_k(\hat{\boldsymbol{L}}) X,
\end{equation}
where $w_k^L$ and $w_k^H$ are computed by substituting $\gamma$ in Eq.~\eqref{Eq: w_k} with $\gamma^L$ and $\gamma^H$ from Eq.~\eqref{Eq: gamma_i}, respectively. The resulting node embeddings $\boldsymbol{Z}_L, \boldsymbol{Z}_H \in \mathbb{R}^{n \times d}$ correspond to low- and high-pass embeddings.
Importantly, the Chebyshev basis terms $T_k(\hat{\boldsymbol{L}}) X$ can be efficiently precomputed and cached via iterative sparse matrix multiplication. 
This design enables mini-batch training without the need to process the entire graph during training, thereby significantly improving scalability for large-scale graphs.

\subsection{Anomaly Context-aware Adaptive Fusion}
After obtaining the low- and high-pass node embeddings, a straightforward approach for GAD is to simply concatenate or average them~\cite{BWGNN}, followed by applying a classification head. 
However, such one-size-fits-all strategies fail to capture the complex interplay between homophilic and heterophilic patterns, especially in the presence of the contrast between normal and anomalous nodes. 
This motivates us to design a node-adaptive fusion strategy that dynamically accounts for node-level homophily disparity. 
Furthermore, prior works~\cite{AdaGNN, GFS} have shown that different feature dimensions contribute unequally to downstream tasks, suggesting that dimension-aware fusion is also crucial. 
To achieve both node- and dimension-wise adaptivity, we introduce a coefficient matrix $\boldsymbol{C} \in \mathbb{R}^{n \times d}$ to fuse the low- and high-pass embeddings as follows:
\begin{equation}\label{Eq: Fusion}
    \boldsymbol{Z} = \boldsymbol{C} \odot \boldsymbol{Z}_L + (1 - \boldsymbol{C}) \odot \boldsymbol{Z}_H,
\end{equation}
where $\odot$ denotes the element-wise product, and $\boldsymbol{Z}$ represents the fused node embeddings. 

A naive method for learning $\boldsymbol{C}$ is to use an MLP that takes node features as input~\cite{NFGNN, MGFNN}. 
However, nodes with diverse structural contexts may share similar features, rendering this method ineffective. 
To address this, we propose incorporating each node's local structural context (e.g., $k$-hop subgraph) to guide the learning of $\boldsymbol{C}$. 
Nonetheless, directly using a full $k$-hop neighborhood may introduce irrelevant entities and relations, thereby diluting the anomaly signal. 
This raises the question: \emph{What constitutes a suitable subgraph for GAD?} 
Recent studies~\cite{BWGNN, GHRN} have observed a right-shift phenomenon in the spectral energy of abnormal nodes, indicating a shift from low to high frequencies. 
This spectral energy accumulation can be measured using the Rayleigh Quotient~\cite{RQ}, defined as:
\begin{equation}\label{Eq: RQ}
    RQ(\boldsymbol{x}, \boldsymbol{L}) = \frac{\boldsymbol{x}^\top \boldsymbol{L} \boldsymbol{x}}{\boldsymbol{x}^\top \boldsymbol{x}} = \frac{\sum_{i=1}^{n} \lambda_i \hat{x}_i^2}{\sum_{i=1}^{n} \hat{x}_i^2} = \frac{\sum_{i,j} A_{ij}(x_j - x_i)^2}{\sum_{i=1}^{n} x_i^2},
\end{equation}
where $\boldsymbol{x} \in \mathbb{R}^{n}$ is a graph signal and $\hat{\boldsymbol{x}} = \boldsymbol{U}^\top \boldsymbol{x}$ is the spectrum of $\boldsymbol{x}$. 
By definition, the Rayleigh Quotient captures the spectral energy and reflects structural diversity. 
Specifically, the presence of relation camouflage, a key characteristic of anomalies, tends to increase the Rayleigh Quotient. 
This connection is summarized by the following lemma~\cite{BWGNN}:
\begin{lemma}
    The Rayleigh Quotient $RQ(\boldsymbol{x}, \boldsymbol{L})$, which quantifies the accumulated spectral energy of a graph signal, is monotonically increasing with the node's anomaly degree.
\end{lemma}

In other words, for a given node, the subgraph with the highest Rayleigh Quotient is likely to contain the most anomaly-relevant structural information. 
Guided by this insight, we utilize MRQSampler~\cite{UniGAD} to extract a $k$-hop subgraph $\mathcal{G}^{RQ}$ that maximizes the Rayleigh Quotient for each node, ensuring the resulting context is rich in anomaly cues. 
Finally, we generate the node-specific fusion coefficients $\boldsymbol{C}$ using an MLP conditioned on both the input features $\boldsymbol{X}$ and the anomaly context $\mathcal{G}^{RQ}$:
\begin{equation}\label{Eq: C}
    \boldsymbol{C} = \operatorname{MLP}(\operatorname{AGG}(\mathcal{G}^{RQ}, \boldsymbol{X}) \| \boldsymbol{X}),
\end{equation}
where $\|$ denotes concatenation and $\operatorname{AGG}(\cdot)$ is a parameter-free pooling function (e.g., mean pooling) that summarizes the anomaly-aware subgraph context for each node. 
This design of Anomaly Context-aware Adaptive Fusion enables the model to assign appropriate fusion weights to each node and feature dimension based on their specific structural and attribute properties, effectively mitigating node-level homophily disparity and improving detection performance.

\subsection{Frequency Preference Guidance Loss}
As revealed by the class-level local homophily disparity, abnormal nodes tend to camouflage themselves and exhibit higher heterophily. 
This suggests that different types of nodes prefer different frequency components: while normal nodes are typically associated with low-frequency signals, anomalies, which behave distinctly from their neighbors, tend to exhibit high-frequency characteristics. 
To model this class-specific frequency preference, we introduce a regularization term to guide the learning of the coefficient matrix $\boldsymbol{C}$. 
Let $c_i = \frac{1}{d} \sum_{j=1}^d \boldsymbol{C}_{ij}$ denote the average low-frequency preference of node $i$, computed by averaging over the embedding dimensions. 
We encourage $c_i$ to approach a predefined target $p^a \in [0,1]$ for abnormal nodes and $p^n \in [0,1]$ for normal nodes, with the constraint $p^a \le p^n$ to enforce a stronger high-frequency preference in anomalies. 
The frequency preference guidance loss is defined as a binary cross-entropy loss:
\begin{equation}\label{Eq: FPG}
\begin{aligned}
    \mathcal{L}_{fpg} = 
    &-\frac{\beta}{|\mathcal{V}^L|} \sum_{i \in \mathcal{V}^L, y_i = 1} \left( p^a \log c_i + (1 - p^a) \log (1 - c_i) \right) \\
    &-\frac{1}{|\mathcal{V}^L|} \sum_{i \in \mathcal{V}^L, y_i = 0} \left( p^n \log c_i + (1 - p^n) \log (1 - c_i) \right),
\end{aligned}
\end{equation}
where $\beta$ is the ratio of abnormal to normal nodes in the labeled set, used to counteract class imbalance.
This regularization enforces a frequency-aware separation between normal and abnormal nodes by aligning their fusion coefficients with distinct spectral preferences. 
Specifically, the constraint $p^a \le p^n$ biases the model to allocate more high-frequency information to anomalies, thereby enhancing the model's sensitivity to class-level homophily disparity and improving anomaly detection performance.

\vpara{Overall Objective.}
Using the fused embeddings $\boldsymbol{Z}$ obtained from Eq.~\eqref{Eq: Fusion}, we employ an MLP to predict the label $\hat{y}_i$ for each labeled node $i$, and compute the binary cross-entropy loss:
\begin{equation}\label{Eq: BCE}
    \mathcal{L}_{bce} = -\frac{1}{|\mathcal{V}^L|} \sum_{i \in \mathcal{V}^L} \left( \beta y_i \log \hat{y}_i + (1 - y_i) \log (1 - \hat{y}_i) \right).
\end{equation}
The final training objective combines the binary classification loss and the frequency preference guidance loss:
\begin{equation}\label{Eq: Overall}
    \mathcal{L} = \mathcal{L}_{bce} + \mathcal{L}_{fpg}.
\end{equation}

\subsection{Theoretical Analysis}
In this subsection, we provide a theoretical foundation to support our architectural design. Our analysis is grounded in the Contextual Stochastic Block Model (CSBM)~\cite{CSBM}, a widely adopted generative model for attributed graphs~\cite{GCTheory, HomoNece, Demy}. 
To properly reflect homophily disparity, degree heterogeneity, and class imbalance in GAD, we first introduce a variant:

\begin{definition}[$CSBM(n_a, n_n, \boldsymbol{\mu}, \boldsymbol{\nu}, (p_1,q_1),(p_2,q_2), \boldsymbol{\theta})$]\label{Def: ASBM}
Let $\mathcal{C}_a$ and $\mathcal{C}_n$ be the abnormal and normal node sets with sizes $n_a$ and $n_n$, respectively; $n=n_a+n_n$, class priors $\pi_a=n_a/n$, $\pi_n=1-\pi_a$ with $\pi_a\ll \pi_n$. 
Each row of the feature matrices $\boldsymbol{X}_a \in \mathbb{R}^{n_a \times d}$ and $\boldsymbol{X}_n \in \mathbb{R}^{n_n \times d}$ is sampled from $\mathcal{N}(\boldsymbol{\mu}, \frac{1}{d}\boldsymbol{I})$ and $\mathcal{N}(\boldsymbol{\nu}, \frac{1}{d}\boldsymbol{I})$, where $|\boldsymbol{\mu}|_2, |\boldsymbol{\nu}|_2 \leq 1$.
The full feature matrix $\boldsymbol{X}$ is formed by stacking $[\boldsymbol{X}_a; \boldsymbol{X}_n]$. 
Each node $v_{i}$ has a degree parameter $\theta_i>0$ (collected in $\boldsymbol{\theta}$), which controls how many edges it tends to form.  Each node adopts a connectivity regime: homophilic or heterophilic: a node in the homophilic regime prefers to connect to nodes of the same class, while a node in the heterophilic regime prefers the opposite. 
Accordingly, each potential edge $(i,j)$ is generated as
\[
\mathbb{P}(A_{ij}=1 \mid y_i, y_j, h_i)= \theta_i \theta_j \, \mathbf{B}^{(h_i)}_{y_i,y_j},
\]
where $y_i\in\{a,n\}$ is the class label and $h_i\in\{o,e\}$ is the regime of node $v_{i}$. 
The pattern-specific base matrices are
\[
\begin{aligned}
&\mathbf{B}^{(o)}=\begin{pmatrix} p_1 & q_1 \\ q_1 & p_1 \end{pmatrix},\quad p_1>q_1 \ \text{(homophily)};
\\
&\mathbf{B}^{(e)}=\begin{pmatrix} p_2 & q_2 \\ q_2 & p_2 \end{pmatrix},\quad p_2<q_2 \ \text{(heterophily)}.
\end{aligned}
\]
Here, $p_1>q_1$ enforces homophily (intra-class edges more likely), while $p_2<q_2$ enforces heterophily (inter-class edges more likely). 
We denote the corresponding homophilic and heterophilic node sets as $\mathcal{H}_o$ and $\mathcal{H}_e$. 
\end{definition}

In line with prior analyses~\cite{GCTheory}, we consider a linear classifier parameterized by $\boldsymbol{w} \in \mathbb{R}^d$ and $b \in \mathbb{R}$. 
For each node, the predicted label is computed as $\hat{\boldsymbol{y}} = \sigma(\hat{\boldsymbol{X}}\boldsymbol{w} + b\boldsymbol{1})$, where $\sigma(\cdot)$ is the sigmoid function, and $\hat{\boldsymbol{X}} = \boldsymbol{U}g(\boldsymbol{\Lambda})\boldsymbol{U}^\top \boldsymbol{X}$ denotes the spectrally filtered features. 
We adopt the binary cross-entropy loss defined in Eq.~\eqref{Eq: BCE}. 
The following theorem characterizes the model's separability when node-adaptive filters are applied:


\begin{theorem}\label{The: Node-wise Filtering}
For a graph $\mathcal{G}(\mathcal{V}, \mathcal{E}, \boldsymbol{X}) \sim CSBM(n_a, n_n, \boldsymbol{\mu}, \boldsymbol{\nu}, (p_1,\\ q_1), (p_2, q_2), \boldsymbol{\theta})$, when low- and high-pass filters are applied separately to the homophilic and heterophilic node sets $\mathcal{H}_o, \mathcal{H}_e$, there exist parameters $\boldsymbol{w}^*, b^*$ such that all nodes are linearly separable with probability $1-o_d(1)$.
\end{theorem}

\input{tables/auprc}

The proof is provided in Appendix~\ref{App: Proof}. 
Theorem~\ref{The: Node-wise Filtering} theoretically guarantees that, under mild assumptions, applying node-wise low- and high-pass filters according to node homophily enables linear separability of the filtered features.
Our architectural design is built upon this principle. 
We first extract candidate embeddings for each node using both low- and high-pass filters. 
Then, our Anomaly Context-aware Adaptive Fusion module adaptively combines the filtered features conditioned on node-specific attributes and structural contexts. 
Guided by both classification and frequency preference loss, the model dynamically adjusts fusion coefficients to reflect each node's homophily. 
This learned fusion strategy approximates the node-wise filter assignment postulated in the theorem, thereby approaching the theoretical bound of linear separability and enhancing performance.

\subsection{Complexity Analysis}
The overall framework consists of two stages: a preprocessing phase and a mini-batch training phase.
In the preprocessing stage, we compute the $K$-hop Chebyshev basis embeddings $T_k(\hat{\boldsymbol{L}}) \boldsymbol{X}$ via iterative sparse matrix multiplications, which requires $\mathcal{O}(K(m + n))$ time. 
Additionally, we extract the Rayleigh Quotient-guided subgraph $\mathcal{G}^{RQ}$ using the MRQSampler, with a total complexity of $\mathcal{O}(n \log n)$~\cite{UniGAD}.
Since the sampling procedure for each node is independent, this step can be parallelized to further reduce runtime. 
Notably, all preprocessing operations are performed only once and are reused during both training and inference, thus incurring minimal overhead.
In the mini-batch training phase, both the Chebyshev embeddings $T_k(\hat{\boldsymbol{L}})\boldsymbol{X}$ and the parameter-free  $\operatorname{AGG}(\mathcal{G}^{RQ}, \boldsymbol{X})$ are precomputed and cached. 
As a result, each node can be treated independently during training, and the encoder, which consists of simple MLPs, exhibits time complexity that scales linearly with the batch size.
In summary, with the decoupled architecture, SAGAD demonstrates strong scalability and is well-suited for large-scale graphs like T-Social, which contains over 5.78M nodes and 73.1M edges (please see "Scalability Comparison" in Section~\ref{Sec: Model Comparison}).

\section{Experiments}

In this section, we conduct comprehensive experiments and answer the following research questions: 
\textbf{RQ1:} How does SAGAD compare to state-of-the-art GAD baselines? 
\textbf{RQ2:} To what extent can SAGAD mitigate homophily disparity?  
\textbf{RQ3:} How well does SAGAD scale on large-scale graphs?  
\textbf{RQ4} Do our key components contribute positively to GAD? 
\textbf{RQ2:} How to visually understand the learned fusion coefficients? 
\textbf{RQ6:} How do hyperparameters influence the performance of SAGAD?

\subsection{Experimental Setup}
\vpara{Datasets.}
Following GADBench~\cite{GADBench}, we conduct experiments on 10 real-world datasets spanning various scales and domains: Reddit~\cite{RedWei}, Weibo~\cite{RedWei}, Amazon~\cite{Amazon}, YelpChi~\cite{YelpChi}, T-Finance~\cite{BWGNN}, Elliptic~\cite{Elliptic}, Tolokers~\cite{TolQue}, Questions~\cite{TolQue},  DGraph-Fin~\cite{DGraph}, and T-Social~\cite{BWGNN}. 
Dataset statistics are presented in Table \ref{Tab: Dataset statistics}.

\vpara{Baselines.}
We compare our model with a series of baseline methods, which can be categorized into the following groups: 
standard GNNs, including GCN~\cite{GCN}, SAGE~\cite{SAGE}, GIN~\cite{GIN}, GAT~\cite{GAT}; 
heterophilic GNNs, including H2GCN~\cite{H2GCN}, FAGCN~\cite{FAGCN}, AdaGNN~\cite{AdaGNN}, BernNet~\cite{BernNet}, and ACM~\cite{ACM}; GAD-specific GNNs, including GAS~\cite{GAS}, DCI~\cite{DCI}, PCGNN~\cite{PCGNN}, AMNet~\cite{AMNet}, BWGNN~\cite{BWGNN}, GHRN~\cite{GHRN}, XGBGraph~\cite{GADBench}, ConsisGAD~\cite{ConsisGAD}, and SpaceGNN~\cite{SpaceGNN}. 

\vpara{Metrics and Implementation Details.}
Following GADBench~\cite{GADBench}, we use the Area Under the Receiver Operating Characteristic Curve (AUROC), the Area Under the Precision-Recall Curve (AUPRC) calculated via average precision, and the Recall score within the top-$K$ predictions (Rec@K) as metrics, with 100 labeled nodes (20 anomalies) per training set. 
All experiments are averaged over 10 random splits provided by GADBench for robustness. 
Due to space constraints, more details about evaluation protocols and hyperparameter settings are provided in Appendix~\ref{App: Evaluation Protocols} and Appendix~\ref{App: Implementation Details}, respectively.
Our code is publicly available at \url{https://github.com/Cloudy1225/SAGAD}.

\begin{table*}[h]
\centering
\caption{Performance gaps (AUPRC difference) between Q1 and other quartiles across datasets. Lower absolute values indicate smaller homophily disparity.}
\label{Tab: Homo Variance}
\begin{tabular}{l|ccc|ccc|ccc|ccc}
\toprule
\multirow{2}{*}{Dataset} & \multicolumn{3}{c|}{Weibo} & \multicolumn{3}{c|}{T-Finance} & \multicolumn{3}{c|}{Amazon} & \multicolumn{3}{c}{YelpChi} \\
\cmidrule(lr){2-4} \cmidrule(lr){5-7} \cmidrule(lr){8-10} \cmidrule(lr){11-13}
 & Q1-Q2 & Q1-Q3 & Q1-Q4 & Q1-Q2 & Q1-Q3 & Q1-Q4 & Q1-Q2 & Q1-Q3 & Q1-Q4 & Q1-Q2 & Q1-Q3 & Q1-Q4 \\
\midrule
AMNet     & \emph{9.40} & \emph{32.88} & \textbf{27.32} & \emph{3.44} & 30.27 & \textbf{58.66} & 5.43 & 20.65 & 47.10 & \textbf{-0.71} & \textbf{-0.19} & 0.54 \\
BWGNN     & 10.73 & 39.18 & 44.89 & 9.47 & 56.56 & 86.60 & 8.89 & 26.14 & 56.74 & \emph{-1.09} & \emph{-0.60} & \emph{0.43} \\
ConsisGAD & 13.43 & 43.95 & 54.96 & 3.75 & \textbf{25.04} & 89.25 & \textbf{-3.24} & \textbf{14.63} & \textbf{39.47} & -1.75 & \textbf{-0.19} & \textbf{-0.06} \\
SAGAD     & \textbf{1.95} & \textbf{7.37} & \emph{35.26} & \textbf{1.46} & \emph{26.26} & \emph{70.22} & \emph{5.42} & \emph{14.95} & \emph{46.93} & 4.30 & 6.25 & 7.07 \\
\bottomrule
\end{tabular}
\end{table*}

\subsection{Model Comparison}\label{Sec: Model Comparison}

\vpara{Performance Comparison (RQ1).}
Table~\ref{Tab: AUPRC} presents the AUPRC performance of all compared methods. 
Additional results for AUROC and Rec@K are provided in Appendix Table~\ref{Tab: AUROC} and Table~\ref{Tab: Rec@K}, respectively. 
Overall, our method consistently outperforms all standard GNNs, heterophilic GNNs, and GAD-specific GNNs, achieving state-of-the-art results.
Specifically, SAGAD yields an average improvement of 5.0\% in AUPRC, 3.3\% in AUROC, and 5.7\% in Rec@K over the strongest baselines, ConsisGAD, XGBGraph, and SpaceGNN. 
On Weibo and T-Social, SAGAD achieves substantial AUPRC gains of 8.8\% and 12.8\%, respectively. 
These results strongly support that our model better captures subtle anomaly cues in GAD.
However, we also observe a performance gap between SAGAD and XGBGraph on the Elliptic dataset. 
This may be attributed to the highly heterogeneous tabular node features in Elliptic, which are better handled by tree-based models such as XGBoost compared to deep learning methods~\cite{Tree}. 
Nevertheless, SAGAD still outperforms all other deep learning baselines by a considerable margin.

\begin{figure}[t]
    \centering
    \subfigure[Amazon]{\includegraphics[width=0.495\linewidth]{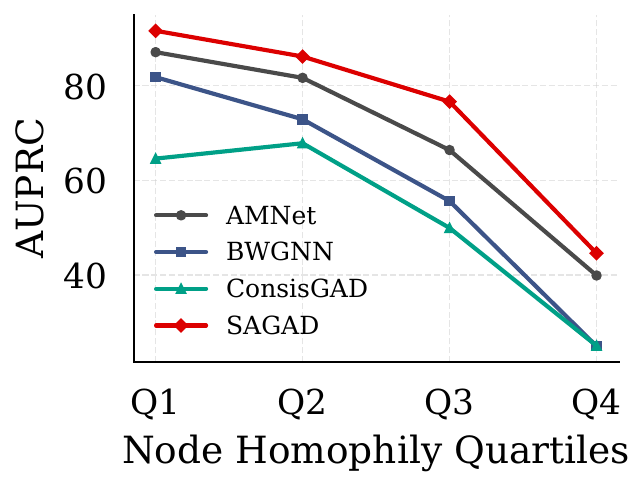}\label{Fig: Perf_Disparity_Amazon}}
    \subfigure[YelpChi]{\includegraphics[width=0.495\linewidth]{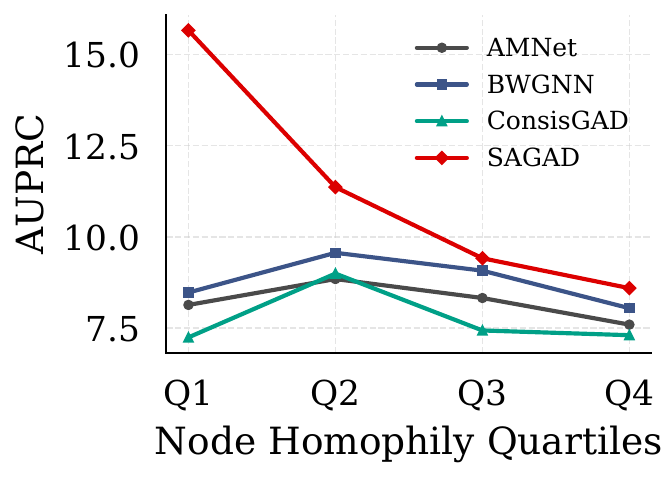}\label{Fig: Perf_Disparity_Yelp}}
    \caption{Performance disparity across node homophily quartiles (Q1 = top 25\% homophily, Q4 = bottom 25\%) on Amazon and YelpChi.}
    \label{Fig: Perf_Disparity}
\end{figure}

\vpara{Homophily Disparity Mitigation (RQ2).}
Local homophily has a profound impact on GNN-based methods due to their intrinsic reliance on neighborhood message passing, which inevitably propagates noisy signals in heterophilic regions. 
As a result, the performance of GNNs tends to decrease when node homophily becomes lower, making homophily disparity a fundamental challenge rather than one that can be completely eliminated. 
To examine how different approaches cope with this issue, we divide the abnormal nodes in the test set into four groups based on their node homophily and assess performance within each group against the remaining normal nodes. 
As shown in Figures~\ref{Fig: Homo_Perf_Disparity} and~\ref{Fig: Perf_Disparity}, most existing methods tend to perform better on nodes with high homophily compared to those with lower homophily. 
In contrast, SAGAD consistently improves performance across all homophily quartiles, demonstrating enhanced robustness to homophily disparity. 
We further report the AUPRC gaps between Q1 and the other quartiles (Q2--Q4) in Table~\ref{Tab: Homo Variance}, where SAGAD attains the lowest or second-lowest variance on the majority of datasets. 
Even on YelpChi, where the variance is larger than baselines (due to their weak performance in Q1), SAGAD delivers a substantial AUPRC improvement in the high-homophily group Q1 (+7.19\%). 
These results indicate that our anomaly context-aware adaptive fusion, combining Rayleigh Quotient-guided subgraphs with node-specific fusion, effectively mitigates, though does not fully eliminate, the negative effects of homophily disparity.

\vpara{Scalability Comparison (RQ3).}
We evaluate both time and memory efficiency during training on YelpChi and T-Social. 
Table~\ref{Tab: Scalability Comparison} shows that our method significantly reduces both training time and GPU memory usage. 
On the large-scale T-Social dataset (5.78M nodes and 73.1M edges), SAGAD requires only 1455.50 MB of GPU memory, approximately 10× less than competitive baselines, demonstrating excellent scalability. 
This efficiency stems from our decoupled architecture, which enables mini-batch training without processing the entire graph, thereby ensuring linear scalability with respect to batch size.

\input{tables/scalability}

\subsection{Model Analyses}

\input{tables/ablation_auprc}

\vpara{Ablation Study (RQ4).}
We conduct ablation experiments to investigate the contributions of different components in our framework. 
Specifically, we consider the following variants: (i) using only the low-pass filter or only the high-pass filter to derive node embeddings; (ii) replacing our adaptive fusion with alternatives, including mean, concatenation, or attention-based fusion~\cite{AMNet}, to integrate low- and high-pass embeddings; (iii) learning the coefficient matrix $\boldsymbol{C}$ in Eq.~\eqref{Eq: C} with only raw node features $\boldsymbol{X}$ or with the full $k$-hop subgraph \ding{66}, instead of our Rayleigh Quotient-guided subgraph; and (iv) removing the frequency preference guidance loss $\mathcal{L}_{fpg}$. 
The results are summarized in Table~\ref{Tab: Ablation Study AUPRC AUROC}.

The results reveal several important findings.
(1) Variants that exploit both low- and high-frequency information consistently outperform those relying on only one component. This confirms that both homophilic and heterophilic cues are indispensable for GAD tasks. 
(2) Our anomaly context-aware adaptive fusion achieves the best performance compared to mean, concatenation, and attention-based fusion strategies, underscoring the benefits of jointly considering frequency-selective embeddings and node-specific adaptivity. 
(3) The Rayleigh Quotient-guided subgraph proves more effective than using either the full $k$-hop subgraph or only the raw node features. This highlights its ability to capture localized structural anomalies while filtering out irrelevant neighborhood noise. 
(4) The regularization term $\mathcal{L}_{fpg}$ provides additional improvements by explicitly aligning the spectral preferences of anomalies and normal nodes, thereby enhancing their discriminability. 
Overall, these results validate that each proposed design, dual-pass filtering, adaptive fusion, Rayleigh Quotient-guided context, and frequency preference regularization, contributes synergistically to the superior performance of our framework.

\vpara{Coefficients Visualization (RQ5).}
We further provide heatmaps in Figure~\ref{Fig: Coef_Heatmap} to illustrate the learned fusion coefficients.
For clarity, we visualize the top 8 dimensions of $\boldsymbol{C}$ for randomly selected abnormal and normal nodes.
As shown, the fusion coefficients are personalized on a per-node basis, confirming that the model learns node-specific weighting schemes. 
Moreover, abnormal nodes generally exhibit lower coefficient values than normal nodes, which aligns with the intuition that anomalies should rely more on high-frequency information to enhance distinguishability.
Overall, our model enables fine-grained fusion, allowing each node to emphasize relevant frequency components in its embedding based on local anomaly structures. 

\begin{figure}[t]
    \centering
    \subfigure[YelpChi]{\includegraphics[width=0.495\linewidth]{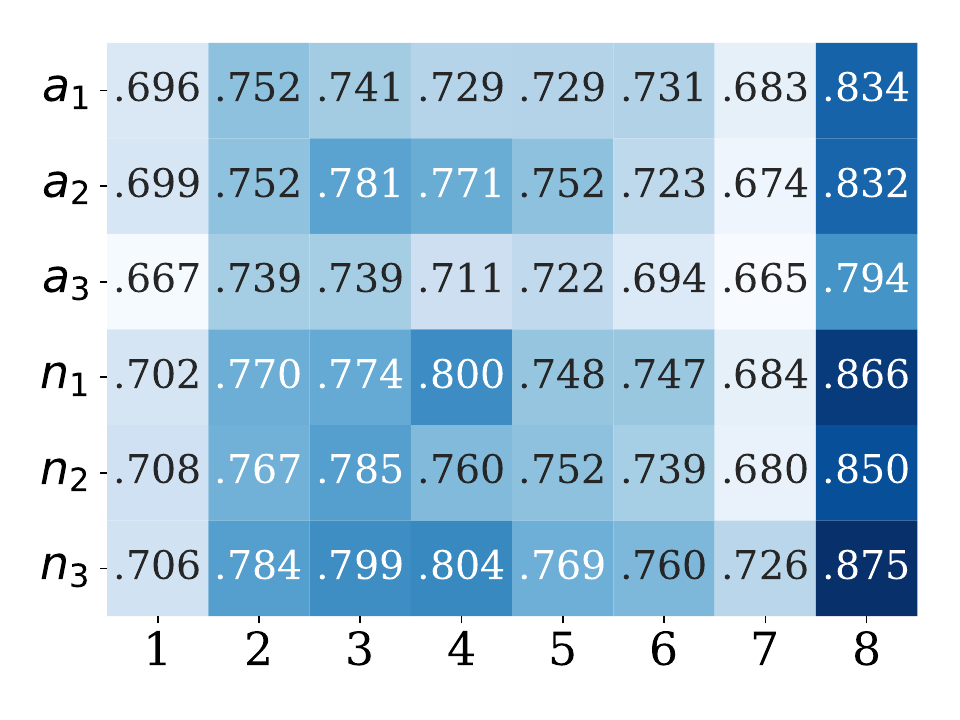}}
    \subfigure[T-Finance]{\includegraphics[width=0.495\linewidth]{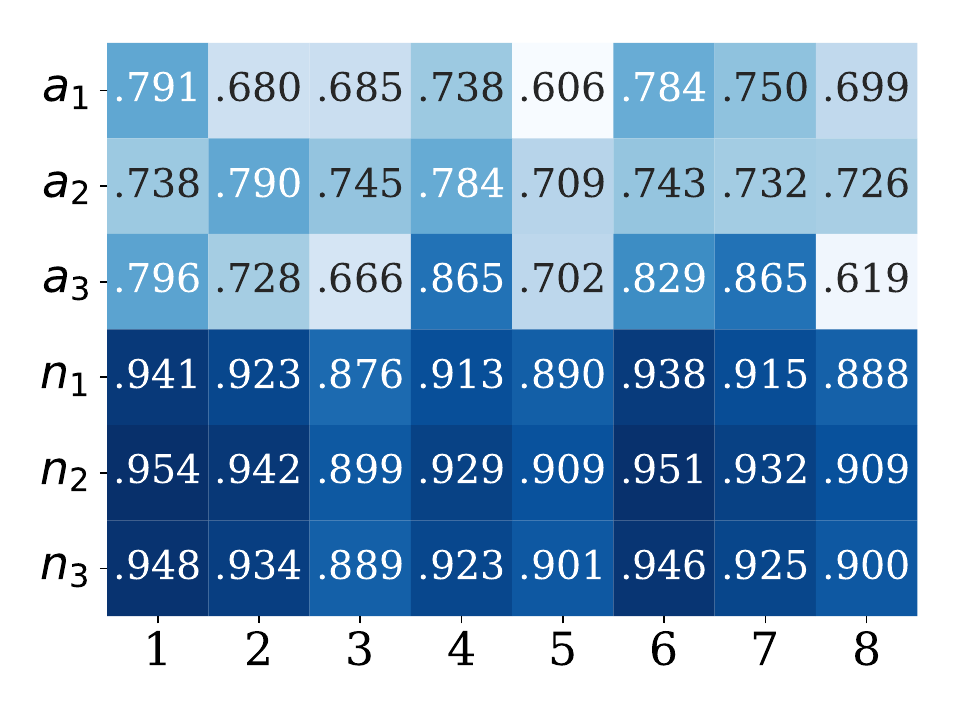}}
    \caption{Visualization of the learned coefficients for the top 8 dimensions. Nodes ($a_1, a_2, a_3$) and ($n_1, n_2, n_3$) are randomly selected abnormal and normal nodes, respectively.}
    \label{Fig: Coef_Heatmap}
\end{figure}

\vpara{Hyperparameter Analysis (RQ6).}
We analyze the sensitivity of our model to two key hyperparameters, $p_a$ and $p_n$, which guide the expected preference for low-frequency information in abnormal and normal nodes.
We vary both values from 0.0 to 1.0 and report the AUPRC in Figure~\ref{Fig: Varying_p_AUPRC}. 
Results show that higher $p_n$ (right side of the heatmap) generally yields better performance, as normal nodes benefit more from low-frequency signals that reflect global smoothness. 
The best results appear when $p_a \leq p_n$, indicating that abnormal nodes favor more high-frequency information, consistent with their irregular behavior.
Overall, SAGAD shows robustness to $p_a$ and $p_n$: by maintaining $p_a$ and $p_n$ within the appropriate range, it achieves competitive performance.




\section{Conclusion}
This paper tackles homophily disparity and scalability issues in GAD by introducing SAGAD. 
It integrates three key designs: a dual-pass Chebyshev polynomial filter to capture both low- and high-frequency signals, an anomaly context-aware adaptive fusion to selectively combine them based on Rayleigh Quotient-guided subgraphs, and a frequency preference guidance loss to encode class-specific spectral biases. 
Our theoretical analysis established that node-adaptive filtering ensures asymptotic linear separability. 
Additionally, SAGAD consistently outperformed state-of-the-art baselines, demonstrated stronger robustness under homophily disparity, and achieved substantial reductions in memory and time consumption on large-scale graphs.

\begin{acks}
This work is supported by the National Natural Science Foundation of China (62306137).
\end{acks}

\bibliographystyle{ACM-Reference-Format}
\bibliography{main}

\appendix

\input{tables/dataset}

\section{Proof}\label{App: Proof}
\begin{proof}
We extend the argument of~\cite{GCTheory} from a single-pattern $CSBM$ to the mixed-pattern, degree-corrected, and class-imbalanced $ASBM$ in Definition~\ref{Def: ASBM}.
Throughout the proof, we make the following standard assumptions:
\begin{enumerate}
    \item the graph size is relatively large with $\omega(d\log d)\le n\le \mathcal{O}(poly(d))$;
    \item sparsity level obeys $p_1,q_1,p_2,q_2=\omega(\log^2 n/n)$;
    \item degree parameters are bounded and centered as $\theta_{\min}\le \theta_i\le \theta_{\max}$ with $\chi:=\theta_{\max}/\theta_{\min}=\mathcal{O}(1)$ and $\mathbb{E}[\theta_i\mid z_i]=1$ for $z_i\in\{a,n\}$;
    \item class prior is imbalanced with $\pi_a=n_a/n\ll \pi_n=1-\pi_a$.
\end{enumerate}

Following~\cite{GCTheory, HomoNece, Demy, Node-MoE}, we adopt the random-walk operator $\boldsymbol{S}=\boldsymbol{D}^{-1}\boldsymbol{A}$ as a low-pass filter and its negative $-\boldsymbol{S}$ as a high-pass filter.
Concretely, each node $v_{i}$ is assigned a spectral filter function $g(\cdot)$ depending on its pattern: 
$g(\lambda)=\lambda$ if $v_{i} \in \mathcal{H}_o$ (homophilic), and $g(\lambda)=-\lambda$ if $v_{i} \in \mathcal{H}_e$ (heterophilic).
This yields the node-adaptive filtered features
$\hat{\boldsymbol{X}}=\boldsymbol{U}g(\boldsymbol{\Lambda})\boldsymbol{U}^\top\boldsymbol{X}$.

\paragraph{Concentration under degree correction.}
Let degree $d_i=\sum_j A_{ij}$ and recall that, conditional on $(z_i,h_i)$, we have
$\mathbb{E}[A_{ij}\mid z_i,z_j,h_i]=\theta_i\theta_j \mathbf{B}^{(h_i)}_{z_i,z_j}$.
Standard Bernstein~\cite{Bernstein1, Bernstein2} bounds with bounded $\theta$ and $\omega(\log^2 n / n)$ sparsity yield
\[
d_i = \theta_i\, n \big(\pi_a \beta_{z_i,a}^{(h_i)} + \pi_n \beta_{z_i,n}^{(h_i)}\big)\,(1\pm o(1)),
\]
where $\beta_{u,v}^{(h)}:=\mathbf{B}^{(h)}_{u,v}$.
Hence for any node $v_{i}$ and any unit vector $\boldsymbol{w}$,
\[
\begin{aligned}
&\hat{\boldsymbol{X}}_i
= \pm\,\frac{1}{d_i}\sum_{j} A_{ij}\, \boldsymbol{X}_j
= \pm\Big(\mathbb{E}[\hat{\boldsymbol{X}}_i \mid z_i,h_i] + \boldsymbol{\xi}_i\Big), 
\\
&\big|\langle \boldsymbol{\xi}_i,\boldsymbol{w}\rangle\big| 
= \mathcal{O}\!\left(\sqrt{\frac{\log n}{d\, d_i}}\right)
= \mathcal{O}\!\left(\sqrt{\frac{\log n}{d\, n\, \theta_i\, \kappa_i}}\right),
\end{aligned}
\]
with $\kappa_i:=\pi_a \beta_{z_i,a}^{(h_i)} + \pi_n \beta_{z_i,n}^{(h_i)}$. 
The sign ± corresponds to low- vs. high-pass filtering. 
Let the effective average connectivity under
imbalance be
\[
\kappa_{\mathrm{eff}}
:= \min\big\{\pi_a p_1+\pi_n q_1,\ \pi_a q_1+\pi_n p_1,\ \pi_a p_2+\pi_n q_2,\ \pi_a q_2+\pi_n p_2\big\}.
\]
Using $\theta_i\ge \theta_{\min}$ and $\kappa_i\ge \kappa_{\mathrm{eff}}$, we obtain the uniform deviation bound~\cite{GCTheory}
\begin{equation}
\label{eq:dev-ASBM}
\big|\langle \hat{\boldsymbol{X}}_i - \mathbb{E}[\hat{\boldsymbol{X}}_i\mid z_i,h_i], \boldsymbol{w}\rangle\big|
= \mathcal{O}\!\left(\sqrt{\frac{\log n}{d\, n\, \kappa_{\mathrm{eff}}}}\right),
\end{equation}
matching the CSBM rate up to the effective factor $\kappa_{\mathrm{eff}}$.

\paragraph{Pattern-dependent means after node-adaptive filtering.}
Condition on $z_i$ and $h_i$.
A neighboring node $v_{j}$ belongs to class $a$ with probability proportional to $\theta_j \beta^{(h_i)}_{z_i,a}$ and to class $n$ with probability proportional to $\theta_j \beta^{(h_i)}_{z_i,n}$.
By (iii) and bounded heterogeneity $\chi=\mathcal{O}(1)$, the neighbor-class proportions concentrate at
\[
\omega_{i,a}^{(h_i)}=\frac{\pi_a \beta^{(h_i)}_{z_i,a}}{\pi_a \beta^{(h_i)}_{z_i,a}+\pi_n \beta^{(h_i)}_{z_i,n}},
\quad
\omega_{i,n}^{(h_i)}=1-\omega_{i,a}^{(h_i)}.
\]
Thus the low-pass mean is a degree-normalized mixture
\[
\mathbb{E}[\boldsymbol{S}\boldsymbol{X}]_i
= \omega_{i,a}^{(h_i)} \boldsymbol{\mu} + \omega_{i,n}^{(h_i)} \boldsymbol{\nu}\ (1\pm o(1)),
\]
while the high-pass mean flips the sign:
$\mathbb{E}[-\boldsymbol{S}\boldsymbol{X}]_i = -\mathbb{E}[\boldsymbol{S}\boldsymbol{X}]_i$.
Enumerating cases gives, for $h_i=o$ (homophily):
\[
\mathbb{E}[\hat{\boldsymbol{X}}_i]
=
\begin{cases}
\frac{\pi_a p_1 \boldsymbol{\mu} + \pi_n q_1 \boldsymbol{\nu}}{\pi_a p_1+\pi_n q_1}\ (1\pm o(1)), & z_i=a,\\[4pt]
\frac{\pi_a q_1 \boldsymbol{\mu} + \pi_n p_1 \boldsymbol{\nu}}{\pi_a q_1+\pi_n p_1}\ (1\pm o(1)), & z_i=n,
\end{cases}
\]
and for $h_i=e$ (heterophily) the same expressions with $(p_1,q_1)$ replaced by $(p_2,q_2)$, followed by an overall sign flip due to the high-pass ($-\boldsymbol{S}$).
Consequently, under the node-adaptive choice (low-pass on $\mathcal{H}_o$ and high-pass on $\mathcal{H}_e$), all class-wise means align in the same ordering along direction $\boldsymbol{\nu}-\boldsymbol{\mu}$:
\[
\langle \mathbb{E}[\hat{\boldsymbol{X}}_i\mid z_i=a],\boldsymbol{\nu}-\boldsymbol{\mu}\rangle
<
\langle \mathbb{E}[\hat{\boldsymbol{X}}_i\mid z_i=n],\boldsymbol{\nu}-\boldsymbol{\mu}\rangle,
\]
with a gap proportional to $\Delta_{h}:=\frac{|p_h-q_h|}{\pi_a p_h+\pi_n q_h + \pi_a q_h+\pi_n p_h}$ (here $h\in\{1,2\}$ indexes $(p_1,q_1)$ or $(p_2,q_2)$), hence lower bounded by a constant multiple of $\frac{|p_h-q_h|}{p_h+q_h}$ and independent of $\theta$ due to the random-walk normalization.

\paragraph{Prior-aware linear separator and margin.}
Consider the linear classifier with direction
$\boldsymbol{w}^* = R\frac{\boldsymbol{\nu}-\boldsymbol{\mu}}{\|\boldsymbol{\mu}-\boldsymbol{\nu}\|}$ and a bias that accounts for the class prior shift
\[
b^* = -\frac{\langle \boldsymbol{\mu}+\boldsymbol{\nu}, \boldsymbol{w}^*\rangle}{2}
\ +\ \tau_\pi,
\qquad
\tau_\pi = R\,\frac{\log(\pi_a/\pi_n)}{\|\boldsymbol{\mu}-\boldsymbol{\nu}\|}.
\]
The additive term $\tau_\pi$ is the standard LDA correction under unequal priors with spherical covariances~\cite{HastieTF}.
For any $v_{i} \in \mathcal{C}_a$, combining Step~2 and Eq~\ref{eq:dev-ASBM} gives
\begin{align*}
&\langle \hat{\boldsymbol{X}}_i,\boldsymbol{w}^* \rangle + b^*\\
&= \langle \mathbb{E}[\hat{\boldsymbol{X}}_i\mid z_i=a],\boldsymbol{w}^* \rangle 
 - \frac{\langle \boldsymbol{\mu}+\boldsymbol{\nu},\boldsymbol{w}^*\rangle}{2} + \tau_\pi
 + \underbrace{\langle \hat{\boldsymbol{X}}_i-\mathbb{E}[\hat{\boldsymbol{X}}_i],\boldsymbol{w}^*\rangle}_{\text{fluctuation}} \\
&= -\frac{R}{2}\,\Gamma_{h_i}\,\|\boldsymbol{\mu}-\boldsymbol{\nu}\| \ (1\pm o(1))\ +\ \tau_\pi\ +\ \mathcal{O}\!\Big(R\sqrt{\tfrac{\log n}{d\, n\, \kappa_{\mathrm{eff}}}}\Big),
\end{align*}
where $\Gamma_{h_i}>0$ depends on the pattern (homophily vs.\ heterophily) via the mixture coefficients in Step~2 (and is uniformly bounded away from $0$ as $|p_h-q_h|>0$).
For $i\in\mathcal{C}_n$ the sign of the leading term is positive.
Hence, if the feature center distance satisfies
\[
\|\boldsymbol{\mu}-\boldsymbol{\nu}\|
\ \ge\ C\,\frac{\log n}{\sqrt{d\, n\, \kappa_{\mathrm{eff}}}}
\]
for a sufficiently large constant $C>0$, the margin term dominates both the stochastic fluctuation
$\mathcal{O}\!\big(R\sqrt{\tfrac{\log n}{d\, n\, \kappa_{\mathrm{eff}}}}\big)$
and the prior correction $\tau_\pi= \mathcal{O}\!\big(R\,|\log(\pi_a/\pi_n)|/\|\boldsymbol{\mu}-\boldsymbol{\nu}\|\big)$ even when $\pi_a\ll\pi_n$.
Therefore, according to part 2 of Theorem 1 in \cite{GCTheory}, $\operatorname{sign}(\langle \hat{\boldsymbol{X}}_i,\boldsymbol{w}^*\rangle + b^*)$ equals the true label for all nodes with probability $1-o_d(1)$, by a union bound over $v_{i} \in \mathcal{V}$.
This establishes linear separability of node-adaptively filtered features under degree correction and class imbalance, proving the theorem.
\end{proof}

\begin{figure}[t]
    \centering
    \subfigure[Reddit]{\includegraphics[width=0.495\linewidth]{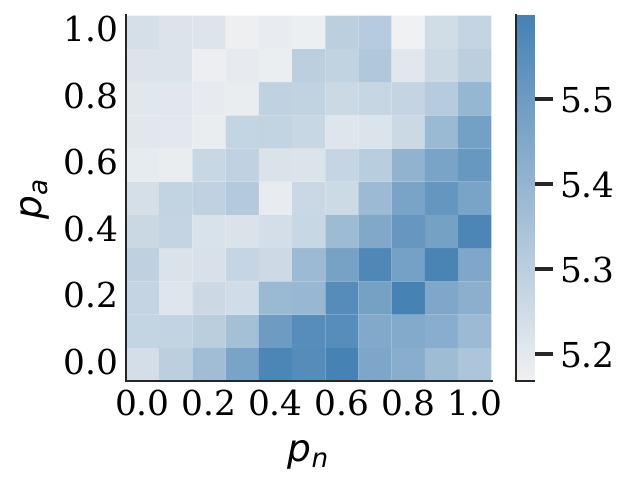}}
    \subfigure[YelpChi]{\includegraphics[width=0.495\linewidth]{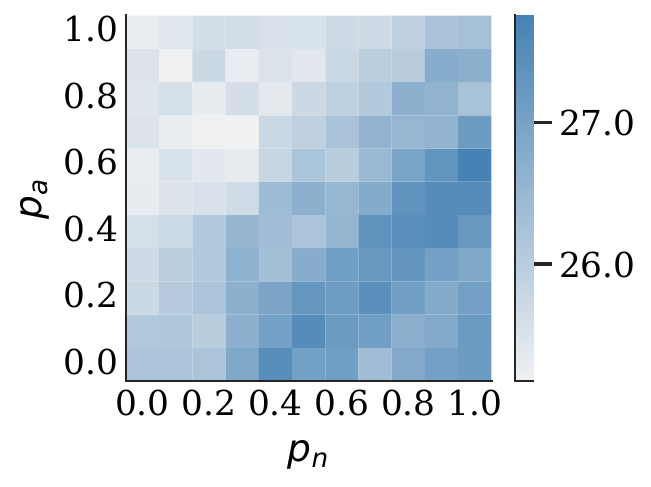}}
    \subfigure[Tolokers]{\includegraphics[width=0.495\linewidth]{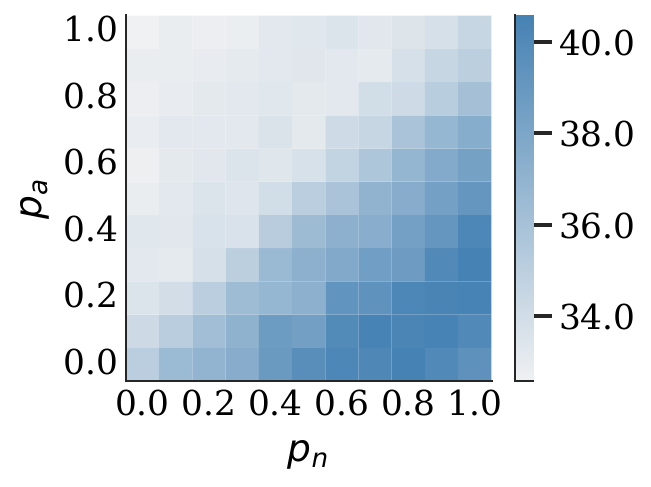}}
    \subfigure[T-Social]{\includegraphics[width=0.495\linewidth]{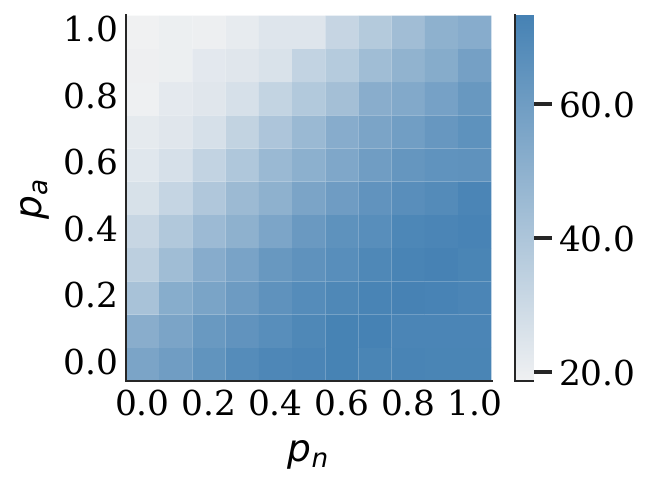}}
    \caption{How the AUPRC score varies with different values of $p_a$ and $p_n$.}
    \label{Fig: Varying_p_AUPRC}
\end{figure}

\section{Additional Experimental Details}


\input{tables/auroc}

\input{tables/reck}

\subsection{Evaluation Protocols}\label{App: Evaluation Protocols}
Following GADBench~\cite{GADBench}, we evaluate performance using three popular metrics: the Area Under the Receiver Operating Characteristic Curve (AUROC), the Area Under the Precision-Recall Curve (AUPRC) calculated via average precision, and the Recall score within the top-$K$ predictions (Rec@K). 
Here, $K$ corresponds to the number of anomalies in the test set. 
For all metrics, anomalies are treated as the positive class, with higher scores indicating better model performance. 
To simulate real-world scenarios with limited supervision~\cite{GADBench}, we standardize the training/validation set across all datasets to include 100 labels—20 positive (abnormal) and 80 negative (normal) labels. 
To ensure the robustness of our findings, we perform 10 random splits, as provided by~\cite{GADBench}, on each dataset and report the average performance.

\subsection{Implementation Details}\label{App: Implementation Details}
We utilize the official implementations of all baseline methods provided either by the GADBench repository~\cite{GADBench} or by the original authors. 
All experiments are conducted on a Linux server equipped with an Intel(R) Xeon(R) Gold 6248 CPU running at 2.50GHz and a 32GB NVIDIA Tesla V100 GPU. 
In the preprocessing stage, the Chebyshev polynomial order $K$ is selected from $\{2, 3, 4, 5\}$. 
For subgraph extraction, we adopt 1-hop Rayleigh Quotient-guided subgraphs to balance performance and computational efficiency. 
During mini-batch training, the model is trained for up to 1000 epochs using the Adam optimizer~\cite{Adam}, with a patience of 50. 
Hyperparameters are tuned as follows: 
learning rate $\in \{0.1, 0.01, 0.001\}$, 
weight decay $\in \{0.0, 0.01, 0.0001\}$, 
hidden dimension $\in \{ 32, 64, 128 \}$, 
dropout rate $\in \{0.0, 0.1, 0.2, 0.3, 0.4, 0.5\}$, 
MLP depth $\in \{1, 2, 3\}$, 
activation function $\in \{ \text{ReLU}, \text{ELU}, \text{Tanh}, \text{Identity} \}$, 
and normalization $\in \{ \text{none}, \text{batch}, \text{layer} \}$. 
The fusion-related hyperparameters $p_n$ and $p_a$ are varied within the range of 0.0 to 1.0. 
Our code is publicly available at \url{https://github.com/Cloudy1225/SAGAD}.

\end{document}

%% file: tables/auprc.tex
\begin{table*}[t]
\centering
\caption{Comparison of AUPRC for each model. "-" denotes "out of memory".} \label{Tab: AUPRC}
\begin{tabular}{l|cccccccccc|c}\toprule
Model & Reddit   & Weibo   & Amazon   & Yelp.   & T-Fin.   & Ellip.   & Tolo.   & Quest.   & DGraph.   & T-Social   & Avg.\\\midrule
GCN & 4.2±0.8 & \underline{86.0±6.7} & 32.8±1.2 & 16.4±2.6 & 60.5±10.8 & 43.1±4.6 & 33.0±3.6 & 6.1±0.9 & 2.3±0.2 & 8.4±3.8 & 29.3 \\
GIN & 4.3±0.6 & 67.6±7.4 & 75.4±4.3 & 23.7±5.4 & 44.8±7.1 & 40.1±3.2 & 31.8±3.2 & 6.7±1.1 & 2.0±0.1 & 6.2±1.7 & 30.3 \\
SAGE & 4.5±0.6 & 58.5±6.2 & 42.5±6.1 & 20.9±3.5 & 11.7±5.2 & 43.1±5.6 & 34.0±2.1 & 5.5±1.3 & 2.0±0.2 & 7.8±1.3 & 23.0 \\
GAT & 4.7±0.7 & 73.3±7.3 & 81.6±1.7 & 25.0±2.9 & 28.9±8.6 & 44.2±6.6 & 33.0±2.0 & 7.3±1.2 & 2.2±0.2 & 9.2±2.0 & 30.9 \\
\midrule
H2GCN & 4.6±0.7 & 76.5±5.8 & 80.9±2.6 & 24.4±3.7 & 34.2±10.1 & 46.8±8.1 & 33.6±2.2 & 7.4±1.3 & 2.4±0.3 & 10.4±1.9 & 32.1 \\
ACM & 4.4±0.7 & 66.0±8.7 & 54.0±19.0 & 21.4±2.7 & 29.2±16.8 & 55.1±4.8 & 34.4±3.5 & 7.2±1.9 & 2.2±0.4 & 6.0±1.6 & 28.0 \\
FAGCN & 4.7±0.7 & 70.1±10.6 & 77.0±2.3 & 22.5±2.6 & 39.8±27.2 & 43.6±10.6 & 35.0±4.3 & 7.3±1.4 & 2.0±0.3 & - & - \\
AdaGNN & \underline{4.9±0.8} & 28.3±2.8 & 75.7±6.3 & 22.7±2.1 & 23.3±7.6 & 39.2±7.9 & 32.2±3.9 & 5.3±0.9 & 2.1±0.3 & 4.8±1.1 & 23.9 \\
BernNet & \underline{4.9±0.3} & 66.6±5.5 & 81.2±2.4 & 23.9±2.7 & 51.8±12.4 & 40.0±4.1 & 28.9±3.5 & 6.7±2.1 & \underline{2.5±0.2} & 4.2±1.2 & 31.1 \\
\midrule
GAS & 4.7±0.7 & 65.7±8.4 & 80.7±1.7 & 21.7±3.3 & 45.7±13.4 & 46.0±4.9 & 31.7±3.0 & 6.3±2.0 & \underline{2.5±0.2} & 8.6±2.4 & 31.4 \\
DCI & 4.3±0.4 & 76.2±4.3 & 72.5±7.9 & 24.0±4.8 & 51.0±7.2 & 43.4±4.9 & 32.1±4.2 & 6.1±1.3 & 2.0±0.2 & 7.4±2.5 & 31.9 \\
PCGNN & 3.4±0.5 & 69.3±9.7 & 81.9±1.9 & 25.0±3.5 & 58.1±11.3 & 40.3±6.6 & 33.9±1.7 & 6.4±1.8 & 2.4±0.4 & 8.0±1.6 & 32.9 \\
AMNet & \underline{4.9±0.4} & 67.1±5.1 & 82.4±2.2 & 23.9±3.5 & 60.2±8.2 & 33.3±4.8 & 28.6±1.5 & 7.4±1.4 & 2.2±0.3 & 3.1±0.3 & 31.3 \\
BWGNN & 4.2±0.7 & 80.6±4.7 & 81.7±2.2 & 23.7±2.9 & 60.9±13.8 & 43.4±5.5 & 35.3±2.2 & 6.5±1.7 & 2.1±0.3 & 15.9±6.2 & 35.4 \\
GHRN & 4.2±0.6 & 77.0±6.2 & 80.7±1.7 & 23.8±2.8 & 63.4±10.4 & 44.2±5.7 & \underline{35.9±2.0} & 6.5±1.7 & 2.3±0.3 & 16.2±4.6 & 35.4 \\
XGBGraph & 4.1±0.5 & 75.9±6.2 & \underline{84.4±1.1} & 24.8±3.1 & 78.3±3.1 & \textbf{77.2±3.2} & 34.1±2.8 & 7.7±2.1 & 1.9±0.2 & 40.6±7.6 & \underline{42.9} \\
ConsisGAD & 4.5±0.5 & 64.6±5.5 & 78.7±5.7 & \underline{25.9±2.9} & 79.7±4.7 & 47.8±8.2 & 33.7±2.7 & \underline{7.9±2.4} & 2.0±0.2 & 41.3±5.0 & 38.6 \\
SpaceGNN & 4.6±0.5 & 79.2±2.8 & 81.1±2.3 & 25.7±2.4 & \underline{81.0±3.5} & 44.1±3.5 & 33.8±2.5 & 7.4±1.6 & 2.0±0.3 & \underline{59.0±5.7} & 41.8 \\
\midrule
\textbf{SAGAD} & \textbf{5.6±0.6} & \textbf{94.8±0.7} & \textbf{84.7±2.0} & \textbf{27.8±2.4} & \textbf{82.8±1.0} & \underline{57.1±4.1} & \textbf{40.5±1.6} & \textbf{11.0±1.4} & \textbf{2.6±0.3} & \textbf{71.8±8.3} & \textbf{47.9} \\
\bottomrule
\end{tabular}
\end{table*}

%% file: tables/scalability.tex
\begin{table}[ht]
    \centering
    \setlength{\tabcolsep}{3pt}
    \caption{
    Scalability comparison in terms of training time and GPU memory.
    }
    \label{Tab: Scalability Comparison}
    \begin{tabular}{lcccc}
        \toprule
        \multirow{2}{*}{Model} & \multicolumn{2}{c}{YelpChi} & \multicolumn{2}{c}{T-Social} \\
        \cmidrule(lr){2-3} \cmidrule(lr){4-5}
                               & Time (s) & Mem. (MB)        & Time (s) & Mem. (MB) \\
        \midrule
        GCN        & 1.93  & 547.73   & 61.75   & 13241.26 \\
        AMNet      & 4.09  & 773.38   & 213.82  & 18970.04 \\
        BWGNN      & 2.66  & 729.17   & 112.01  & 16146.46 \\
        GHRN       & 3.22  & 3360.71  & 161.00  & 28080.33 \\
        ConsisGAD  & 89.28 & 16390.42 & 1674.87 & 14145.61 \\
        SpaceGNN   & 13.80 & 24217.29 & 1273.43 & 20518.91 \\
        \midrule
        \textbf{SAGAD}      & \textbf{1.52}  & \textbf{74.72}    & \textbf{13.76}   & \textbf{1455.50}  \\
        \bottomrule
    \end{tabular}
\end{table}

%% file: tables/ablation_auprc.tex
\begin{table*}[t]
  \centering
  \caption{Ablation study on each component of our SAGAD in terms of AUPRC and AUROC.}
  \label{Tab: Ablation Study AUPRC AUROC}
  \resizebox{1.\textwidth}{!}{
  \begin{tabular}{ccccc|cc|cc|cc|cc|cc|cc|cc}
  \toprule
    \multirow{2}{*}{$\boldsymbol{Z}_L$} & \multirow{2}{*}{$\boldsymbol{Z}_H$} & \multirow{2}{*}{$\mathcal{G}^{RQ}$} & \multirow{2}{*}{$\mathcal{L}_{fpg}$} & \multirow{2}{*}{Fusion} 
    & \multicolumn{2}{c|}{Reddit} 
    & \multicolumn{2}{c|}{Amazon} 
    & \multicolumn{2}{c|}{Yelp.} 
    & \multicolumn{2}{c|}{T-Fin.} 
    & \multicolumn{2}{c|}{Ellip.} 
    & \multicolumn{2}{c|}{Tolo.} 
    & \multicolumn{2}{c}{DGraph.} \\ 
    \cmidrule(lr){6-7}\cmidrule(lr){8-9}\cmidrule(lr){10-11}\cmidrule(lr){12-13}
    \cmidrule(lr){14-15}\cmidrule(lr){16-17}\cmidrule(lr){18-19}
     & & & & & AUPRC & AUROC & AUPRC & AUROC & AUPRC & AUROC & AUPRC & AUROC & AUPRC & AUROC & AUPRC & AUROC & AUPRC & AUROC \\
  \midrule
    \ding{51} & \ding{51} & \ding{51} & \ding{51} & Eq. (7) 
      & \textbf{5.6±0.6} & \textbf{66.5±3.0}
      & \textbf{84.7±2.0} & \textbf{94.5±2.1}
      & \textbf{27.8±2.4} & \textbf{68.8±2.5}
      & \textbf{82.8±1.0} & \textbf{94.4±0.4}
      & \textbf{57.1±4.1} & \textbf{90.9±0.9}
      & \textbf{40.5±1.6} & \textbf{73.6±0.8}
      & \textbf{2.6±0.3} & \textbf{70.8±1.9} \\
    \ding{51} & \ding{51} & \ding{51} & \ding{55} & Eq. (7) 
      & 5.3±0.7 & 63.9±3.6
      & 83.8±2.4 & 94.3±1.9
      & 26.6±1.7 & 68.1±2.5
      & 78.1±6.1 & 93.2±2.1
      & 56.6±6.0 & 90.5±1.1
      & 35.0±2.0 & 69.1±1.8
      & 2.3±0.2 & 69.6±1.3 \\
    \ding{51} & \ding{51} & \ding{66} & \ding{51} & Eq. (7) 
      & 5.5±0.7 & 65.6±2.9
      & 83.0±2.3 & 93.8±2.0
      & 23.1±2.8 & 63.9±4.3
      & 78.1±4.3 & 93.2±1.3
      & 51.9±3.9 & 90.1±1.2
      & 33.7±1.3 & 68.2±1.7
      & 2.2±0.2 & 69.0±1.3 \\
    \ding{51} & \ding{51} & \ding{55} & \ding{51} & Eq. (7) 
      & 5.2±0.7 & 63.8±3.7
      & 83.2±2.0 & 93.9±2.1
      & 23.1±2.8 & 63.8±4.3
      & 79.1±3.6 & 93.3±1.3
      & 51.7±3.3 & 90.2±1.2
      & 34.0±1.4 & 68.8±1.2
      & 2.2±0.2 & 69.1±1.5 \\
    \ding{51} & \ding{51} & \ding{55} & \ding{55} & Mean 
      & 5.1±0.7 & 66.0±2.4
      & 82.3±3.3 & 93.9±2.0
      & 23.1±3.0 & 64.1±4.8
      & 67.1±0.9 & 88.2±1.3
      & 52.5±3.8 & 90.4±0.7
      & 33.2±1.2 & 68.3±1.4
      & 2.2±0.2 & 68.8±1.4 \\
    \ding{51} & \ding{51} & \ding{55} & \ding{55} & Concat 
      & 4.9±0.8 & 61.3±5.0
      & 79.8±9.5 & 93.6±2.7
      & 23.7±3.3 & 64.0±4.1
      & 78.7±3.7 & 93.6±0.9
      & 53.4±6.1 & 89.9±1.3
      & 37.4±1.6 & 70.2±1.8
      & 2.4±0.3 & 70.0±1.9 \\
    \ding{51} & \ding{51} & \ding{55} & \ding{55} & Atten. 
      & 4.9±0.7 & 63.2±3.5
      & 84.3±1.4 & 94.4±2.3
      & 23.1±3.1 & 63.7±4.4
      & 69.8±5.8 & 91.8±1.7
      & 37.7±8.0 & 85.7±1.8
      & 35.9±3.0 & 69.7±1.9
      & 2.4±0.3 & 70.0±2.0 \\
    \ding{51} & \ding{55} & \ding{55} & \ding{55} & - 
      & 4.2±0.7 & 58.0±6.3
      & 32.0±1.8 & 81.8±0.7
      & 16.6±1.3 & 52.8±1.8
      & 77.1±1.8 & 92.4±0.6
      & 45.3±4.0 & 89.1±1.2
      & 38.0±2.1 & 71.0±1.1
      & 2.1±0.3 & 65.1±2.1 \\
    \ding{55} & \ding{51} & \ding{55} & \ding{55} & - 
      & 4.9±0.7 & 60.3±2.9
      & 81.5±3.9 & 93.1±2.1
      & 22.6±2.8 & 63.0±4.0
      & 22.0±3.9 & 78.0±2.9
      & 40.6±4.0 & 85.4±1.1
      & 32.1±2.9 & 65.9±2.6
      & 2.3±0.1 & 69.4±1.0 \\
  \bottomrule
  \end{tabular}
  }
\end{table*}

%% file: tables/dataset.tex
\begin{table*}[h]
    \begin{center}
    {\caption{Dataset statistics including the number of nodes and edges, the node feature dimension, the ratio of anomalies, the homophily ratio, the concept of relations, and the type of node features. $h$, $h^a$, and $h^n$ represent the edge homophily, the average homophily for abnormal nodes, and the average homophily for normal nodes, respectively. \emph{As we can see, $h^a$ is much smaller than $h^n$, indicating that abnormal nodes exhibit significantly lower homophily compared to normal nodes.} "Misc." refers to node features that are heterogeneous attributes, which may include categorical, numerical, and temporal information.}\label{Tab: Dataset statistics}}
    \resizebox{1.\textwidth}{!}{
    \begin{tabular}{lccccccccc}
    \toprule
    Dataset & \#Nodes & \#Edges & \#Feat. & Anomaly & $h$ & $h^a$ & $h^n$ & Relation Concept & Feature Type \\
    \midrule 
    Reddit & 10,984 & 168,016 & 64 & 3.3\% & 0.947 & 0.000 & 0.994 & Under Same Post & Text Embedding \\
    Weibo & 8,405 & 407,963 & 400 & 10.3\% & 0.977 & 0.858 & 0.977 & Under Same Hashtag & Text Embedding \\
    Amazon & 11,944 & 4,398,392 & 25 & 9.5\% & 0.954 & 0.102 & 0.968 & Review Correlation & Misc. Information \\
    YelpChi & 45,954 & 3,846,979 & 32 & 14.5\% & 0.773 & 0.195 & 0.867 & Reviewer Interaction & Misc. Information \\
    T-Finance & 39,357 & 21,222,543 & 10 & 4.6\% & 0.971 & 0.543 & 0.976 & Transaction Record & Misc. Information \\
    Elliptic & 203,769 & 234,355 & 166 & 9.8\% & 0.970 & 0.234 & 0.985 & Payment Flow & Misc. Information \\
    Tolokers & 11,758 & 519,000 & 10 & 21.8\% & 0.595 & 0.476 & 0.679 & Work Collaboration & Misc. Information \\
    Questions & 48,921 & 153,540 & 301 & 3.0\% & 0.840 & 0.111 & 0.922 & Question Answering & Text Embedding \\
    DGraph-Fin & 3,700,550 & 4,300,999 & 17 & 1.3\% & 0.993 & 0.013 & 0.997 & Loan Guarantor & Misc. Information \\
    T-Social & 5,781,065 & 73,105,508 & 10 & 3.0\% & 0.624 & 0.174 & 0.900 & Social Friendship & Misc. Information \\
    \bottomrule
    \end{tabular}
    }
    \end{center}
\end{table*}

%% file: tables/auroc.tex
\begin{table*}[t]
\centering
\caption{Comparison of AUROC for each model. "-" denotes "out of memory".}\label{Tab: AUROC}
\begin{tabular}{l|cccccccccc|c}\toprule
Model & Reddit   & Weibo   & Amazon   & Yelp.   & T-Fin.   & Ellip.   & Tolo.   & Quest.   & DGraph.   & T-Social   & Avg.\\\midrule
GCN & 56.9±5.9 & 93.5±6.6 & 82.0±0.3 & 51.2±3.7 & 88.3±2.5 & 86.2±1.9 & 64.2±4.8 & 60.0±2.2 & 66.2±2.5 & 71.6±10.4 & 72.0 \\
GIN & 60.0±4.1 & 83.8±8.3 & 91.6±1.7 & 62.9±7.3 & 84.5±4.5 & 88.2±0.9 & 66.8±5.2 & 62.2±2.2 & 65.7±1.8 & 70.4±7.4 & 73.6 \\
SAGE & 60.3±3.8 & 81.8±6.5 & 81.4±2.0 & 58.9±4.3 & 68.9±5.5 & 87.4±1.0 & 67.6±4.2 & 61.2±2.9 & 64.8±3.3 & 72.0±2.9 & 70.4 \\
GAT & 60.5±3.9 & 86.4±7.7 & 92.4±1.9 & 65.6±4.0 & 85.0±4.5 & 88.5±2.1 & 68.1±3.0 & 62.3±1.4 & 67.2±1.9 & 75.4±4.8 & 75.1 \\
\midrule
H2GCN & 60.3±4.4 & 87.8±5.7 & 91.4±2.4 & 65.0±4.5 & 86.3±3.9 & 89.3±2.1 & 69.1±2.5 & 61.9±4.2 & 69.0±1.9 & 74.5±3.7 & 75.5 \\
ACM & 60.0±4.3 & 92.5±2.9 & 81.8±7.9 & 61.3±3.8 & 82.2±8.1 & 90.0±0.8 & \underline{69.3±4.2} & 60.8±4.1 & 67.6±5.3 & 68.3±4.9 & 73.3 \\
FAGCN & 60.2±4.1 & 83.4±7.7 & 90.4±1.9 & 62.4±2.9 & 82.6±8.4 & 86.1±3.0 & 68.1±6.6 & 60.8±3.2 & 63.0±3.7 & - & - \\
AdaGNN & 62.0±4.7 & 69.5±4.8 & 90.8±2.2 & 63.2±3.1 & 83.6±2.6 & 85.1±2.8 & 63.3±5.3 & 58.5±4.1 & 67.6±3.7 & 64.7±5.6 & 70.8 \\
BernNet & \underline{63.1±1.7} & 80.1±6.9 & 92.1±2.4 & 65.0±3.7 & 91.2±1.0 & 87.0±1.7 & 61.9±5.6 & 61.8±6.4 & 69.0±1.4 & 59.8±6.3 & 73.1 \\
\midrule
GAS & 60.6±3.0 & 81.8±7.0 & 91.6±1.9 & 61.1±5.2 & 88.7±1.1 & 89.0±1.4 & 62.7±2.8 & 57.5±4.4 & \underline{69.9±2.0} & 72.1±8.8 & 73.5 \\
DCI & 61.0±3.1 & 89.3±5.3 & 89.4±3.0 & 64.1±5.3 & 88.0±3.2 & 88.5±1.3 & 67.6±7.1 & 62.2±2.5 & 65.3±2.3 & 74.2±3.3 & 75.0 \\
PCGNN & 52.8±3.4 & 83.9±8.1 & 93.2±1.2 & 65.1±4.8 & 92.0±1.1 & 87.5±1.4 & 67.4±2.1 & 59.0±4.0 & 68.4±4.2 & 69.1±2.4 & 73.8 \\
AMNet & 62.9±1.8 & 82.4±4.6 & 92.8±2.1 & 64.8±5.2 & 92.6±0.9 & 85.4±1.7 & 61.7±4.1 & 63.6±2.8 & 67.1±3.2 & 53.7±3.4 & 72.7 \\
BWGNN & 57.7±5.0 & 93.6±4.0 & 91.8±2.3 & 64.3±3.4 & 92.1±2.7 & 88.7±1.3 & 68.5±2.7 & 60.2±8.6 & 65.5±3.1 & 77.5±4.3 & 76.0 \\
GHRN & 57.5±4.5 & 91.6±4.4 & 90.9±1.9 & 64.5±3.1 & 92.6±0.7 & 89.0±1.3 & 69.0±2.2 & 60.5±8.7 & 67.1±3.0 & 78.7±3.0 & 76.1 \\
XGBGraph & 59.2±2.7 & \underline{96.4±0.7} & \textbf{94.7±0.9} & 64.0±3.5 & \textbf{94.8±0.6} & \textbf{91.9±1.3} & 67.5±3.4 & 61.4±2.9 & 62.4±4.1 & 85.2±1.8 & 77.8 \\
ConsisGAD & 59.6±2.8 & 85.0±3.7 & 92.3±2.2 & 66.1±3.8 & 94.3±0.8 & 88.6±1.3 & 68.5±2.0 & 65.7±3.9 & 67.1±3.0 & 93.1±1.9 & 78.0 \\
SpaceGNN & 62.3±1.9 & 94.4±0.9 & 91.1±2.5 & \underline{66.8±2.8} & 93.4±1.0 & 88.5±1.2 & 68.9±2.6 & \underline{66.0±1.8} & 63.9±3.7 & \underline{94.7±0.7} & \underline{79.0} \\
\midrule
\textbf{SAGAD} & \textbf{66.5±3.0} & \textbf{98.8±0.3} & \underline{94.5±2.1} & \textbf{68.8±2.5} & \underline{94.4±0.4} & \underline{90.9±0.9} & \textbf{73.6±0.8} & \textbf{69.0±2.8} & \textbf{70.8±1.9} & \textbf{95.3±0.7} & \textbf{82.3} \\
\bottomrule
\end{tabular}
\end{table*}

%% file: tables/reck.tex
\begin{table*}[t]
\centering
\caption{Comparison of Rec@K for each model. "-" denotes "out of memory".} \label{Tab: Rec@K}
\begin{tabular}{l|cccccccccc|c}\toprule
Model & Reddit   & Weibo   & Amazon   & Yelp.   & T-Fin.   & Ellip.   & Tolo.   & Quest.   & DGraph.   & T-Social   & Avg.\\\midrule
GCN & 6.2±2.2 & \underline{79.2±4.3} & 36.9±2.6 & 16.9±3.0 & 60.6±7.6 & 49.7±4.2 & 33.4±3.5 & 9.8±1.2 & 3.6±0.4 & 10.2±8.1 & 30.6 \\
GIN & 4.8±1.9 & 66.5±7.3 & 70.4±5.7 & 26.5±6.1 & 54.4±5.0 & 47.6±3.1 & 33.6±3.0 & 10.3±1.1 & 2.1±0.5 & 5.3±2.9 & 32.2 \\
SAGE & 5.8±0.8 & 63.4±6.0 & 48.0±5.6 & 22.9±3.6 & 18.5±9.4 & 48.2±5.8 & 35.2±2.2 & 8.8±2.5 & 2.5±0.7 & 9.5±2.9 & 26.3 \\
GAT & 6.5±2.3 & 70.2±4.6 & 77.1±1.7 & 28.1±3.4 & 36.2±10.3 & 51.4±5.8 & 35.1±1.8 & 10.9±0.9 & 3.1±0.7 & 11.6±3.0 & 33.0 \\
\midrule
H2GCN & 5.8±2.2 & 71.1±4.9 & 77.5±2.5 & 27.2±4.2 & 43.6±11.7 & 53.9±6.9 & 35.2±2.0 & 11.2±1.9 & 3.4±0.9 & 14.9±2.9 & 34.4 \\
ACM & 5.4±1.8 & 70.7±9.5 & 56.1±14.2 & 23.9±3.8 & 37.2±19.3 & 55.2±3.3 & 35.8±4.2 & 11.4±2.6 & 1.9±0.7 & 8.1±1.8 & 30.6 \\
FAGCN & \underline{7.2±1.9} & 67.8±8.1 & 71.7±3.1 & 25.0±2.8 & 39.6±30.3 & 48.5±11.3 & 35.6±3.7 & 12.3±2.3 & 2.5±0.8 & - & - \\
AdaGNN & 6.3±2.2 & 38.3±3.7 & 74.2±4.0 & 25.6±2.4 & 31.3±11.3 & 46.3±7.4 & 33.6±3.7 & 10.0±2.4 & 1.1±0.4 & 7.9±2.7 & 27.5 \\
BernNet & 6.4±1.5 & 60.9±4.6 & 77.2±2.1 & 26.8±3.1 & 60.5±11.1 & 47.0±4.5 & 30.1±3.8 & 10.3±2.7 & \textbf{3.8±0.6} & 3.3±2.8 & 32.6 \\
\midrule
GAS & 6.6±2.5 & 62.0±6.9 & 77.4±1.7 & 24.6±4.1 & 54.2±9.5 & 51.9±5.2 & 33.0±3.9 & 9.1±2.9 & 3.4±0.4 & 11.5±4.6 & 33.4 \\
DCI & 4.5±1.4 & 68.5±3.5 & 68.3±7.2 & 26.8±5.3 & 58.5±6.3 & 50.0±3.8 & 33.5±5.6 & 9.9±1.9 & 2.3±0.7 & 6.3±6.8 & 32.9 \\
PCGNN & 3.0±2.1 & 65.1±6.6 & 78.0±1.5 & 27.8±3.8 & 63.9±6.3 & 46.5±7.3 & 34.3±1.6 & 10.1±3.9 & \underline{3.7±1.0} & 13.5±3.1 & 34.6 \\
AMNet & 6.8±1.5 & 62.1±4.4 & 77.8±2.3 & 26.6±4.3 & 65.7±6.3 & 37.8±6.7 & 30.5±1.9 & 12.7±2.6 & 2.6±0.8 & 1.6±0.5 & 32.4 \\
BWGNN & 6.0±1.4 & 75.1±3.5 & 77.7±1.6 & 26.4±3.2 & 64.9±11.7 & 49.7±6.1 & 35.5±3.1 & 10.9±3.2 & 3.1±0.8 & 24.3±7.4 & 37.4 \\
GHRN & 6.3±1.5 & 72.4±2.6 & 77.7±1.3 & 26.9±3.1 & 67.7±4.3 & 50.8±4.8 & 36.1±3.1 & 11.1±3.4 & 3.4±0.7 & 24.6±7.0 & 37.7 \\
XGBGraph & 4.9±1.9 & 68.9±5.7 & \underline{78.2±1.5} & 26.8±3.0 & 72.4±3.8 & \textbf{68.9±3.7} & \underline{36.6±3.0} & 10.6±2.9 & 2.5±0.7 & 43.0±7.6 & 41.3 \\
ConsisGAD & 6.3±2.5 & 58.6±4.6 & 77.5±2.8 & 28.7±3.2 & 76.5±4.2 & 50.8±7.8 & 34.8±2.3 & \underline{12.8±3.1} & 1.8±0.5 & 48.5±4.6 & 39.6 \\
SpaceGNN & 6.0±2.0 & 72.2±3.9 & 76.8±2.0 & \underline{28.9±2.4} & \underline{76.6±3.7} & 48.6±4.9 & 35.4±2.5 & 11.5±2.2 & 2.3±0.8 & \underline{63.3±4.0} & \underline{42.2} \\
\midrule
\textbf{SAGAD} & \textbf{8.6±1.7} & \textbf{89.0±1.4} & \textbf{80.8±1.3} & \textbf{31.5±2.7} & \textbf{78.3±1.3} & \underline{59.4±2.7} & \textbf{40.2±1.9} & \textbf{14.5±1.3} & 3.6±1.0 & \textbf{73.5±3.7} & \textbf{47.9} \\
\bottomrule
\end{tabular}
\end{table*}